\documentclass[10pt]{article} 
\usepackage[accepted]{tmlr}

\usepackage{amsmath,amsfonts,bm}

\def\eqref#1{equation~\ref{#1}}

\def\1{\bm{1}}

\DeclareMathAlphabet{\mathsfit}{\encodingdefault}{\sfdefault}{m}{sl}
\SetMathAlphabet{\mathsfit}{bold}{\encodingdefault}{\sfdefault}{bx}{n}

\DeclareMathOperator*{\argmin}{arg\,min}

\usepackage{hyperref}
\usepackage{url}

\usepackage{comment}
\usepackage[utf8]{inputenc} 
\usepackage[T1]{fontenc}    
\usepackage{hyperref}       
\usepackage{url}            
\usepackage{booktabs}       
\usepackage{amsfonts}       
\usepackage{nicefrac}       
\usepackage{microtype}      
\usepackage{xcolor}         

\usepackage{algorithm}
\usepackage{algpseudocode}
\usepackage{xcolor}
\usepackage{mathtools}
\usepackage{subcaption}
\usepackage{float}
\usepackage{amsmath}
\usepackage{amssymb}
\usepackage{thmtools}
\usepackage{thm-restate}

\usepackage{amsthm}
\newtheorem{definition}{Definition}
\newtheorem{theorem}{Theorem}
\newtheorem{lemma}{Lemma}

\newcommand{\norm}[1]{||#1||_2}
\newcommand{\normreg}[1]{||#1||}

\newcommand{\normX}{||\mathcal{X}||_2}
\newcommand{\infnormX}{|\mathcal{X}|}

\newcommand{\normW}{||\mathcal{W}||_2}
\newcommand{\hatw}{\tilde{w}_i}
\newcommand{\covx}{X^{T}X}
\newcommand{\vecx}{\text{vec}(X)}

\newcommand{\etalarge}{\eta-\text{large} \; |\mathcal{Y}|}
\newcommand{\covs}{\frac{1}{s}X_S^{T}X_S}
\newcommand{\wstar}{w_{i*}}
\newcommand{\W}{\mathcal{W}}
\newcommand{\defeq}{\vcentcolon=}
\newcommand{\normY}{|\mathcal{Y}|}
\newcommand{\M}{\mathcal{M}}
\newcommand{\algotext}[1]{\texttt{{#1}}}

\usepackage{xcolor}

\usepackage{xcolor}

\title{PRIMO: Private Regression in Multiple Outcomes}

\author{%
  \name Seth Neel \email sneel@hbs.edu \\
  \addr 
  Harvard Business School
}

\def\month{01}  
\def\year{2025} 
\def\openreview{\url{https://openreview.net/forum?id=34vtRA3Nvu}} 

\begin{document}
\maketitle
\begin{abstract}
 We introduce a new private regression setting we call \emph{Private Regression in Multiple Outcomes} (PRIMO), inspired by the common situation where a data analyst wants to perform a set of $l$ regressions while preserving privacy, where the features $X$ are shared across all $l$ regressions, and each regression $i \in [l]$ has a different vector of outcomes $y_i$. Naively applying existing private linear regression techniques $l$ times leads to a $\sqrt{l}$ multiplicative increase in error over the standard linear regression setting. We apply a variety of techniques including sufficient statistics perturbation ($\algotext{SSP}$) and geometric projection-based methods to develop scalable algorithms that outperform this baseline across a range of parameter regimes. In particular, we obtain \emph{no dependence on l} in the asymptotic error when $l$ is sufficiently large. Empirically, on the task of genomic risk prediction with multiple phenotypes we find that even for values of $l$ far smaller than the theory would predict, our projection-based method improves the accuracy relative to the variant that doesn't use the projection.

\end{abstract}
\section{Introduction}
Linear regression is one of the most fundamental statistical tools used across the applied sciences, for both inference and prediction. In genetics, polygenic risk scores \cite{prs_linear, prs_2} are computed by regressing different phenotypes (observed outcomes of interest like the presence of a trait or disease) onto individual genomic data (SNPs) in order to identify genetic risk factors. In the social sciences, observed societal outcomes like income or marital status might be regressed on a fixed set of demographic features \cite{social}. In many of these cases where the data records  correspond to individuals, there are two aspects of the problem setting that co-occur: 
\begin{itemize}
    \item[\textbf{Aspect 1.}]  The individuals may have a legal or moral right to privacy that has the potential to be compromised by their participation in a study.
    \item[\textbf{Aspect 2.}]  Multiple regressions will be ran using the same set of individual characteristics across each regression with different outcomes, either within the same study or across many different studies.
\end{itemize}
Aspect $1$ has been established as a legitimate concern through both theoretical and applied work. The seminal paper of \cite{homer} showed that the presence of an individual in a genomic dataset could be identified given simple summary statistics about the dataset, leading to widespread concern over the sharing of the results of genomic analyses. In the machine learning setting, where what is being released is a model $w$ trained on the underlying data, there is a long line of research into ``Membership Inference Attacks" \cite{mi_survey, mi}, which given access to $w$ are able to identify which points are in the training set. Over the last decade, differential privacy \cite{privacybook} has emerged as a rigorous solution to the privacy risk posed by Aspect $1.$  In the particular case of linear regression the problem of how to privately compute the optimal regressor has been studied in great detail, which we summarise in Subsection~\ref{subsec:plr}. 
\begin{figure}
    \centering
    \includegraphics[scale=.25]{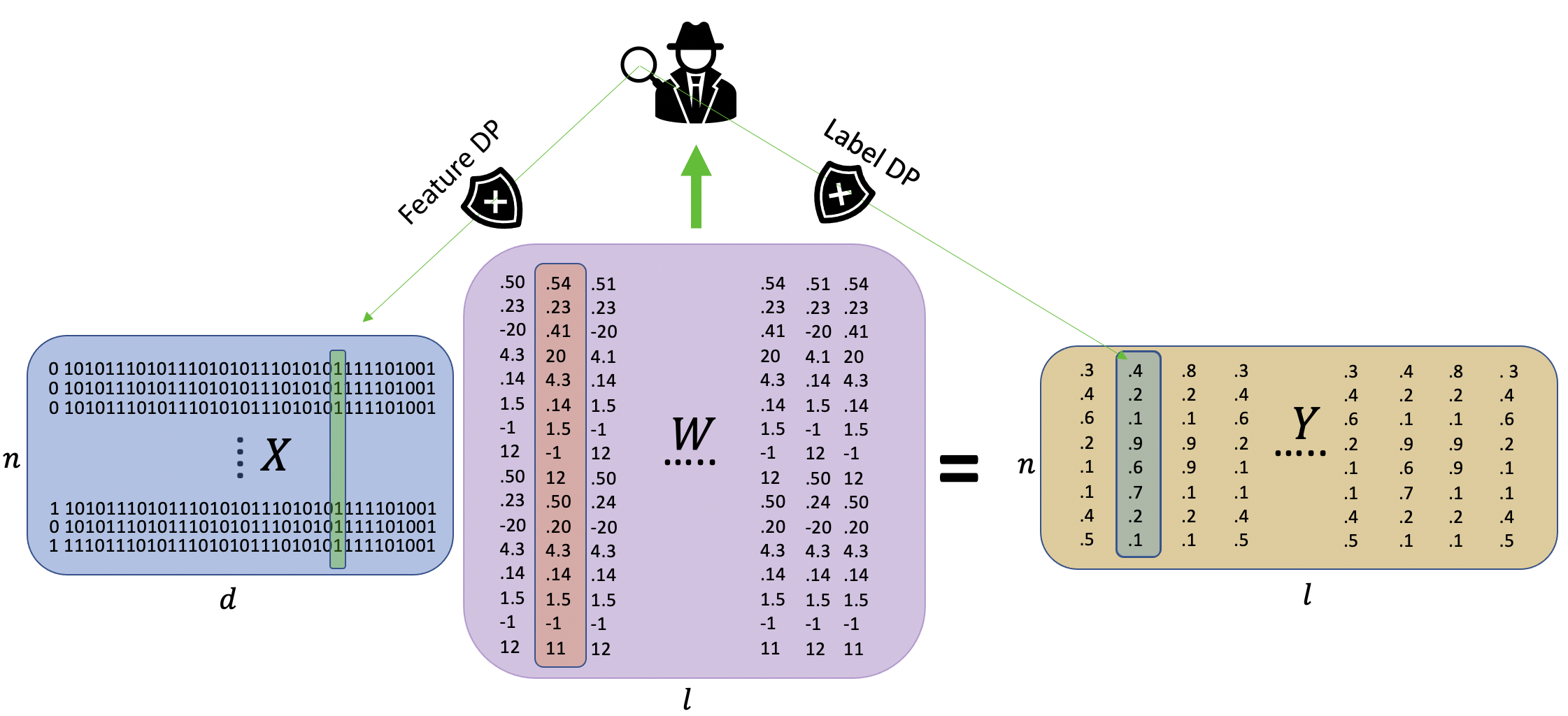}
    \caption {PRIMO imagines the scenario where the weights $W$ computed as a function of sensitive features $X$ and multiple outcomes $Y$ are published or leaked. Our algorithms prevent an adversary with access to $W$ from exposing the underlying sensitive data $x_i, y^i$.}
    \label{fig:my_label}
\end{figure}

Aspect $2$ has been studied extensively from the orthogonal perspective of multiple hypothesis testing, but until now has not been considered in the context of privacy. The problem of overfitting or ``p-hacking'' in the social and natural sciences has been referred to as the ``statistical crisis in science" \cite{gelman}, and developing methods that quantify and mitigate the effects of overfitting has been the subject of much attention in the statistics and computer science communities \cite{ada, ada2, mult}.  
Given the ubiquity of Aspects $1$ and $2$, this raises an important question that we explore in this work: 
\begin{center}
  \textbf{When computing $i = 1 \ldots l$ distinct regressions with a common set of $X$'s and distinct $y_i$'s, what is the optimal accuracy-privacy tradeoff?}
\end{center}


\section{Preliminaries}
We start by defining the standard linear regression problem. Let $X \in \mathcal{X}^{n}\subset \mathbb{R}^{n \times d}$ consist of $d$-dimensional samples from $n$  individuals, $y_i \in \mathcal{Y}^{n}$ be a vector of $n$ outcomes for a regression indexed by $i$, and parameter space $\mathcal{W} \subset \mathcal{R}^d$ denote the subset of linear regression coefficients we will optimize over. 

For $w \in \mathcal{W}$ let $f(w) \defeq \frac{1}{n}\sum_{k = 1}^{n}(w \cdot x_k - y_{ik})^2$ be the linear regression objective, and denote by $\wstar = \argmin_{w \in \mathbb{R}^d}f(w) = (X^{T}X)^{\dagger}X^{T}y_i$, where $\dagger$ is the Moore-Penrose inverse. For $\lambda > 0$ let $f_\lambda(w) = \frac{1}{n}\sum_{k = 1}^{n}(w \cdot x_k - y_{ik})^2 + \lambda \norm{w}^2$ the ridge regression objective, and let $\wstar^\lambda = \argmin f_\lambda(w)_{w \in \mathbb{R}^d} = (\frac{1}{n}X^{T}X + \lambda I_d)^{-1}\frac{1}{n}X^{T}y_i$. 

The definition of data privacy we use throughout is the popular $(\varepsilon, \delta)$-differential privacy introduced in \cite{dwork2}. We refer the reader the to \cite{privacybook} for an overview of the basic properties of $(\varepsilon, \delta)-$DP including closure under post-processing and advanced composition. In order to define a differentially private mechanism, we first define what it means for two datasets to be adjacent. 
\begin{definition}
\label{def:adj}
Two datasets $(X, y), (X', y') \in \mathcal{X}^n \times \mathcal{Y}^n$ are adjacent if they differ in a single element, e.g. there exists $j, j' \in [n]$, such that $X \cup \{x_j'\} \setminus \{x_j\} = X', y \cup \{y_j'\} \setminus \{y_j\} = y'$. We say that $(X, y), (X', y')$ are \emph{feature-adjacent} if they share labels $y=y'$, and $\exists j \in [n]$ such that $X \cup \{x_j'\} \setminus \{x_j\} = X' $. Similarly, they are label-adjacent if $X = X'$, and $\exists j \in [n]$ such that $y \cup \{y_j'\} \setminus \{y_j\} = y'$.    
\end{definition}
We denote adjacency of $(X, y), (X', y')$ by $(X, y) \sim (X', y')$.
\begin{definition}\label{dp}
Let $\M: (\mathcal{X} \times \mathcal{Y})^n \to \mathcal{O}$ a randomized algorithm taking as input a dataset of $n$ records. $\M$ is $(\varepsilon, \delta)$-DP if $\;\forall$ pairs of adjacent datasets $(X, Y) \sim (X', Y'), O \subset \mathcal{O}$:
\begin{equation}\label{eq:dp}
\Pr[\M(X, Y) \in O] \leq e^{\varepsilon}\Pr[\M(X', Y') \in O] + \delta
\end{equation}
\end{definition}
When Definition~\ref{eq:dp} holds for adjacency defined over $(x_i,y_i)$ pairs, this is the standard setting of differential privacy, which we call ``Full DP'' in order to differentiate it from its relaxations. In the less restrictive case where Equation~\ref{eq:dp} holds only over pairs of label-adjacent datasets, we are in the well-studied setting of \emph{label differential privacy} \cite{lab1, lab2, lab3}. When Equation~\ref{eq:dp} holds only over pairs of feature-adjacent datasets, we say that we have \emph{feature differential privacy}. We develop specific PRIMO algorithms for these relaxations of DP in Section~\ref{sec:query_release}. We now formalize the \textit{Private Regression In Multiple Outcomes (PRIMO)} problem. 

\begin{definition}{PRIMO.}
\label{def:primo}
Let $x_i \in \mathcal{X} \subset \mathbb{R}^d, y_{ij} \in \mathcal{Y},$ for $i = 1 \ldots n, j = 1 \ldots l$. Let $X_{n \times d}$ the matrix with $i^{th}$ row $x_i$, and let $Y_{n \times l}$ the matrix with $j^{th}$ column $y_j = (y_{1j}, \ldots y_{nj})$. The optimal solution $W^*$ to the PRIMO problem is
$$ W^{*} = \argmin_{W \in \mathcal{W}^{l} \subset \mathbb{R}^{d \times l}} ||XW - Y||_F^2 $$
\end{definition}
Given a randomized algorithm $\mathcal{M}: (\mathcal{X} \times \mathcal{Y}^{l})^n \to \mathcal{W}^l$, we say that $\mathcal{M}$ is an $(\alpha, \beta, \varepsilon, \delta)$ solution to the PRIMO problem if (i) $\mathcal{M}$ is $(\varepsilon, \delta)$-DP, and (ii) with probability $1-\beta$ over $\tilde{W} \sim \mathcal{M}$: 
$$\frac{1}{nl}||X\tilde{W} - Y||_F^2 - \frac{1}{nl}||XW^* - Y||_F^2 < \alpha$$

We will use $y_j$ to denote the vector of $n$ outcomes for the $j^{th}$ outcome, and $y^i \in \mathcal{Y}^l$ for the vector of $l$ outcomes corresponding to individual $i \in [n]$. We note that in PRIMO under full DP, the adjacency condition holds over pairs $(x_i, y^i)$. The $(\epsilon, \delta)$-DP algorithm we will use most throughout is the Gaussian Mechanism. We state a version of the Gaussian mechanism with constant $c(\epsilon, \delta)$ that is valid for all $\epsilon > 0$,  which follows from analyzing the mechanism using Renyi DP \cite{renyi} and converting back to $(\epsilon, \delta)$-DP.

\begin{lemma}[\cite{privacybook}]
\label{lem:gauss}

Let $f: \mathcal{X}^{n} \to \mathbb{R}^{d}$ an arbitrary $d$-dimensional function, and define it's sensitivity $\Delta_2(f) = \sup_{X \sim X'}\norm{f(X)- f(X')}$, where $X \sim X'$ are datasets that differ in exactly one element. Then the Gaussian mechanism $\algotext{GaussMech}(\varepsilon, \delta, \Delta)$ releases $f(X) + \mathcal{N}(0, \sigma^2$), and is $(\varepsilon, \delta)$-differentially private for $\sigma \geq c(\varepsilon, \delta)\Delta_2(f)$, where $c(\varepsilon, \delta) = \sqrt{2(\frac{1}{\epsilon} + \frac{\log \frac{1}{\delta}}{\varepsilon^2}})$.
\end{lemma}

Lastly, we define some notation. Given a vector $v$ and matrix $A$, $\norm{v}, \norm{A}$ denote the $l_2$ and the spectral norm respectively, $||A||_F$ is the Frobenius norm, $\infnormX = \sup_{x \in \mathcal{X}}||x||_{\infty}, \normY = ||y||_{\infty}$, and $\norm{\hat{w}} = \sqrt{\frac{1}{l}||W^*||^2_F}$, where $W^*$ is the optimal PRIMO solution in Definition~\ref{def:primo}. Following convention in prior work, we will let $\normX, \norm{\mathcal{Y}}, \normW$ denote $\sup_{x \in \mathcal{X}}\norm{x}, \sup_{y \in \mathcal{Y}}\norm{y}, \sup_{w \in \mathcal{W}}\norm{w}$ respectively.

\subsection{Private Linear Regression}
\label{subsec:plr}

Private linear regression is well-studied under a variety of different assumptions on the data generating process and parameter regimes. Typically analysis of private linear regression is done either under the fully agnostic setting where only parameter bounds $\normX, \normY, \normW$ are assumed, or under the assumption of a fixed design matrix and $y$ generated by a linear Gaussian model (the so-called realizable case), or under the assumption of a random design matrix from a known distribution \cite{rand_design}. In this paper we focus on the first fully agnostic setting, because in our intended applications within the social and biomedical sciences in general we neither have realizability (when $y$ is actually a linear function of $x$) or Gaussian features. In the fully agnostic setting \cite{revisit} provides a comprehensive survey of existing private regression approaches and bounds, including a discussion of the impact of different parameter norms.

Broadly speaking, techniques for private linear regression fall into $4$ classes: sufficient statistics perturbation (\algotext{SSP}) \cite{ssp1, ssp2}, Objective Perturbation (\algotext{ObjPert}) \cite{obj}, Posterior sampling \cite{sampling}, and privatized (stochastic) gradient descent (\algotext{NoisySGD}) \cite{sgd}. The methods in this paper are a sub-class of \algotext{SSP}-based methods, which correspond to Algorithm~\ref{alg:reusecov} where $l = 1$. Methods based on $\algotext{SSP}$ rely on perturbing $X^T{X} = \Sigma$ and $X^{T}y_i = \Sigma_{Xy}$ separately with matrices of Gaussian noise (denoted by $E_1, E_{2i}$ respectively), and then using these noisy estimates to compute the (regularized) least squares estimator:
\begin{equation}\label{eq:ssp_decomp}
\tilde{w}_{i*} = \underbrace{(\covx + \lambda I + E_1)^{-1}}_{\text{noisy covariance term } \tilde{\Sigma}} \; \; \; \times \underbrace{(X^{T}y_i + E_{2i})}_{\text{noisy association term } \tilde{\Sigma}_{Xy}}
\end{equation} 

In their Theorem~$5.3$ \cite{priv_lb} give a lower bound on the mean squared error of \emph{any} $(\epsilon, \delta)$-DP algorithm for optimizing a Lipschitz objective function. In the case of linear regression (which corresponds to PRIMO with $l = 1$), if $w_{private}$ denotes the output of any $(\epsilon, \delta)$-DP algorithm, then with probability at least $\frac{1}{3}$:
\begin{equation}
\label{eq:lb}
    \alpha = f(w_{private}) - f(w_{i*}) \geq \min \{\norm{\mathcal{Y}}^2,  \frac{\sqrt{d}(\norm{\mathcal{X}}^2\norm{\mathcal{W}}^2 +  \norm{\mathcal{X}}\norm{\mathcal{W}}\norm{\mathcal{Y}}}{n\epsilon} \}
\end{equation}

\cite{cai2020cost} prove a lower bound that is tailored for realizable low-dimensional ($n \geq d)$ linear regression, where $y \sim x^{t}\beta + N(0, \sigma^2)$, with $\norm{x} \leq 1$ and $||X^{T}X||_{\infty} = O(1/d)$, showing that $\alpha = \Omega(\sigma^2 (\frac{d}{n} + \frac{d^2}{n^2\epsilon^2}))$. Neither of these bounds are exactly applicable to our setting; the bound of \cite{priv_lb} is not neccessarily tight for linear regression, and the bound of \cite{cai2020cost} makes realizability assumptions that clearly fail to hold for the type of genomic applications that motivate our work. We also note that all existing lower bounds for private linear regression are under full DP, and so do not immediately apply to the algorithms in Section~\ref{sec:query_release} which operate in the weaker feature-DP or label-DP settings. In general, the focus of this work is on upper bounds, and so we highlight this curious lack of relevant lower bounds in the literature as a potential future direction, given the long line of work on upper bounds for private linear regression. Acknowledging the above issues with existing bounds, throughout the paper we'll use the lower bound of \cite{priv_lb} as a benchmark for the cost to accuracy of taking $l > 1$, and in settings in which the bounds for PRIMO match this lower bound we will say we have ``PRIMO for Free.'' 

Given any $(\varepsilon, \delta)$-DP algorithm for computing $w_j$ privately, we can use it as a sub-routine to solve PRIMO by simply running it $l$ times to compute each column of $W$. Hence by running any of the optimal algorithms (for example \algotext{SSP} \cite{ssp1}) $l$ times with parameters $\varepsilon' \approx \varepsilon/\sqrt{l}, \delta' \approx \delta/l$, by advanced composition for differential privacy \cite{privacybook} we can achieve MSE, subject to $(\varepsilon, \delta)$-DP:
\begin{equation}\label{eq:naive_baseline}
\alpha = \tilde{O}(\frac{\sqrt{l d}(\norm{\mathcal{X}}^2\norm{\W}^2 + \norm{\mathcal{X}}\norm{\W}\normY}{n\varepsilon}),
\end{equation}
where the $\tilde{O}$ hides factors of $\log(\frac{1}{\delta})$.
So for a fixed privacy budget $\varepsilon$, this naive baseline is a factor of $\sqrt{l}$ worse than in the standard private regression setting where $l = 1$. 

Further background on private query release for linear and low-sensitivity queries (relevant to our techniques in Section~\ref{sec:query_release}) under $l_\infty$ and $l_2$ error, and on sub-sampled linear regression (Subsection~\ref{sec:sub_ssp}) is included in the Appendix.

\section{Results}\label{subsec:results}
The primary contribution of this work is to introduce the novel Private Regression in Multiple Outcomes (PRIMO) problem, and to provide a class of algorithms that trade off accuracy, privacy, and computation. In addition to introducing the PRIMO problem, to our knowledge we are the first to apply private query release methods to linear regression (Section~\ref{sec:query_release}). We also provide a compelling practical application of our algorithms to the problem of genomic risk prediction with multiple phenotypes using data from The $1000$ Genomes Project \cite{1KG} and The Database of Genotypes and Phenotypes \cite{dbgap}.

\begin{table}[!ht]
\centering
     \begin{tabular}{ |p{4cm}|p{3.5cm}|p{1cm}|p{1cm}|p{1cm}|}
     \hline
     \multicolumn{5}{|c|}{Guarantees of \algotext{ReuseCov}} \\
     \hline
     $\quad \quad  \# l$ of regressions & $\quad \quad \alpha = \text{MSE}$ & $\text{cost}(l)$ & $\mathcal{M}$ & Privacy\\
     \hline
     
    $l < \min(\frac{n}{\sqrt{d}}, \frac{\normX^2 \normW^2}{\normY^2})$ & $\tilde{O}\left(\frac{\sqrt{d}\normX^2 \norm{\hat{w}}^2}{n}\right)$ & 1 & \algotext{Gauss} & Full DP \\
    
    $l \in (\frac{\normX^2 \normW^2}{\normY^2},\frac{n}{\sqrt{d}})$ &
    $\tilde{O}\left(\frac{\norm{\hat{w}}\sqrt{ld}\normY \normX}{n} \right)$ & $\sqrt{l}$ & $\algotext{Gauss}$ & Full DP\\
    
    $l > \frac{n}{\sqrt{d}}, n < \frac{\sqrt{d}\normX^2\normW^2}{\normY^2}$   & $\tilde{O}\left(\frac{\sqrt{d}\normX^2 \norm{\hat{w}}^2}{n} \right)$    & 1 & $\algotext{Proj}_Y$ & Feature DP \\
    
    $l > \frac{n}{\sqrt{d}}, n > \frac{\sqrt{d}\normX^2\normW^2}{\normY^2}$   & $\tilde{O}\left(\frac{\norm{\hat{w}}{d}^{1/4}\normY \normX}{\sqrt{n}} \right)$    & $\frac{\sqrt{n}}{d^{1/4}}$ & $\algotext{Proj}_Y$ & Feature DP \\
    
   $\quad \quad \quad \quad l > n^{\frac{2}{3}}$   & $\tilde{O}\left(\norm{\hat{w}}\frac{|\mathcal{X}|\norm{Y}\sqrt{d}}{\sqrt{n}l^{1/4}}\right)$    & $l^{1/4}\sqrt{n}$ & $\algotext{Proj}_X$ & Label DP \\
     \hline
    \end{tabular}
    \caption{In the table above $\text{cost}(l)$ is the ratio of the PRIMO MSE $\alpha$ to the lower bound in Equation~\ref{eq:lb}, and $\mathcal{M}$ denotes the mechanism used to compute the association term $X^{T}Y$ in Algorithm~\ref{alg:reusecov}.}
             \label{tbl:summary}
\end{table}

Algorithm~\ref{alg:reusecov} is our meta-algorithm for the PRIMO problem, and has two variants, \algotext{ReuseCovGauss} which corresponds to $\mathcal{M} = \algotext{GaussMech}$, and $\algotext{ReuseCovProj}$, which corresponds to $\mathcal{M}$ = Algorithm~\ref{alg:proj}. Algorithm~\ref{alg:proj} itself has two variants, which correspond to the label private setting ($\algotext{GaussProj}_X$), and to the feature private setting ($\algotext{GaussProj}_Y$) respectively. For any given setting of the problem parameters ($n, l, d, \normX, \norm{W}, \normY$) one of these variants achieves superior performance, in terms of the upper bound we can prove on the error $\alpha$. Which variant achieves the lowest $\alpha$ depends on the value of $l$ relative to the other parameters, and the type of DP (Full, Feature, or Label) required. In Table~\ref{tbl:summary} we summarize the theoretical upper bound on $\alpha$ achieved by the best PRIMO variant in each setting. In Sections~\ref{sec:ssp}, \ref{sec:query_release} we prove these upper bounds for each variant.

In Section~\ref{sec:ssp} we adapt the previously proposed sufficient statistics perturbation ($\algotext{SSP}$) algorithm of \cite{ssp1, ssp2} into the \algotext{ReuseCovGauss} Algorithm. Via a novel accuracy analysis of $\algotext{SSP}$ for the case when the privacy levels for $E_{1}, E_{2i}$ differ in Equation~\ref{eq:ssp_decomp} (Theorem~\ref{thm:ssp-acc-full}), we show that since the noisy covariance matrix is reused across the regressions (as it only depends on $X$), by allocating the majority of our privacy budget to computing this term, we are able to obtain PRIMO for Free when $l$ is sufficiently small. This happens because the error term is dominated by the error in computing the noisy covariance matrix, which does not depend on $l$, rather than the error from computing the noisy association term (row $1$ of Table~\ref{tbl:summary}). When $l > \frac{\normX^2 \normW^2}{\normY^2}$, the asymptotic error of $\algotext{ReuseCovGauss}$ is dominated by the error in computing the noisy association term, which has a $\sqrt{l}$ dependence in the error on $l$ (row $2$ of Table~\ref{tbl:summary}). Given this, it is natural to ask if, under parameter regimes where the error from the association term dominates, can we obtain improved dependence of $\alpha$ on $l$ over the $\sqrt{l}$ given by $\algotext{ReuseCovGauss}$?

We answer this question in the affirmative in Section~\ref{sec:query_release} with the introduction of the \algotext{ReuseCovProj} algorithms. The key idea is that instead of privately computing $\frac{1}{n}X^{T}Y$ via the Gaussian Mechanism (Lemma~\ref{lem:gauss}), which necessarily has error that scales like $\sqrt{l}$, in the setting where \emph{either} $X$ or $Y$ are not considered private, projecting the noisy value $\frac{1}{n}X^{T}Y + \mathcal{N}(0, \sigma^2)$ returned by the Gaussian Mechanism onto the space of feasible values that the non-private quantity $\frac{1}{n}X^{T}Y$ could take (Line~$5$ of Algorithm~\ref{alg:proj}), can reduce the error for sufficiently large $l$. This set of ``feasible'' values is the image of the domain of $Y$ under $X^{T}$ when $X$ is considered public, and it is the image of the domain of $X$ under $Y^{T}$ when $Y$ is considered public. The assumption that either $X$ or $Y$ are public are critical to these algorithms, as computing the projection depends on the space we are projecting into in a way that would not be easily privatized. The setting where $X$ is public and $Y$ is private corresponds to feature DP, and the setting where $X$ is public and $Y$ is public is label differential privacy. In rows $3-5$ in Table~\ref{tbl:summary} we summarise the guarantees of Theorems~\ref{thm:ssp-query-acc},\ref{thm:ssp-query-acc-label} that characterize the accuracy in these settings. 
In both the Feature DP and Label DP settings, the projection improves the dependence in the error on $l, d$, at the cost of a factor of $\sqrt{n}$ in the denominator. As a result, until $l, d$ are sufficiently large relative to $n$ we'd expect the projection to actually increase error. We discuss these bounds further in Section~\ref{sec:query_release}.

In Section~\ref{sec:experiments}, we implement our $\algotext{ReuseCovGauss}$ and $\algotext{ReuseCovProj}$ algorithms using SNP data from two of the most common genomic databases \cite{dbgap, 1KG}. We compare our algorithms to each other, and to the non-private baseline. In both settings we reach the surprising conclusion that even for relatively small values of $l = 11, 100$, much smaller than the theory would predict, $\algotext{ReuseCovProj}$ outperforms $\algotext{ReuseCovGauss}$. Moreover, $\algotext{ReuseCovProj}$ is able to achieve non-trivial MSE (as measured by $R^2$) for very large values of $l$, as predicted by Theorem~\ref{thm:ssp-query-acc}. We also compare both algorithms to the baseline of running a single private regression $l$ times (via \algotext{DP-SGD} \cite{dp-sgd}) and composing the privacy loss. \algotext{DP-SGD} performs significantly worse than any PRIMO algorithm unless we are in the setting of Row $2$ in Table~\ref{tbl:summary} where we expect no benefit from \algotext{ReuseCovGauss} or \algotext{ReuseCovProj} over \algotext{SSP}, consistent with the theory.


\section{Full DP: The \algotext{ReuseCovGauss} Algorithm}
\label{sec:ssp}


We start by presenting our first algorithm for PRIMO, $\algotext{ReuseCovGauss}$, which corresponds to Algorithm~\ref{alg:reusecov} with $\mathcal{M} = \algotext{GaussMech}$. To understand the intuition behind $\algotext{ReuseCovGauss}$, we start from the analysis of the ridge regression variant of $\algotext{SSP}$ in \cite{revisit}. Equation $13$ in \cite{revisit} shows that if $\tilde{w}_{i}$ is the noisy ridge regressor output by $\algotext{SSP}$, and $\wstar$ is the (non-private) OLS estimator, then w.p. $1-\rho$:
\begin{equation}
\label{eq:ssp_error_terms}
 f(\tilde{w}_{i})-f(\wstar) = \tilde{O}\left(\underbrace{\frac{d}{\lambda \epsilon^2}\normX^2\normY^2}_{\text{association error term}}+ \underbrace{\frac{d}{\lambda\epsilon^2}\normX^4\norm{\W}^2}_{\text{covariance error term}}\right) + \underbrace{\lambda \norm{\W}^2}_{\text{error due to ridge penalty}}
\end{equation}

Inspecting these terms, we see that when $\norm{\mathcal{X}}\norm{\W} \gg \norm{\mathcal{Y}}$ the error is dominated by the covariance error term, which is the error due to random noise injected when privately computing $X^{T}X$, rather than the association error term, which is error due to random noise injected when privately computing $X^{T}y_i$. We now turn back to our PRIMO setting, and imagine independently applying $\algotext{SSP}$ to solve each of our $i = 1\ldots l$ regression problems. By the lower bounds in \cite{priv_lb}, given a fixed privacy budget this incurs at least a $\sqrt{l}$ multiplicative blow up in error. However, we notice that when our private linear regression subroutine is an $\algotext{SSP}$ variant, this naive scheme of running $l$ independent copies of $\algotext{SSP}$ is grossly wasteful. Since the $X$ matrix is shared across all the regressions, we can simply compute our noisy estimate of $X^{T}X$ once, and then reuse it across all the regressions. Combining these two observations, we present Algorithm~\ref{alg:reusecov}. $\M$ can be any $(\epsilon/2, \delta/2)$-DP algorithm for estimating $X^{T}Y$, but when $M = \algotext{GaussMech}$ we call Algorithm~\ref{alg:reusecov} $\algotext{ReuseCovGauss}$.

\begin{algorithm}[!h]{$\algotext{ReuseCov}$}
\caption{Input:  $n, \lambda, X \in \mathcal{X}^n \subset \mathbb{R}^{d \times n}$, $Y = [y_1, \ldots y_{l}] \in \mathcal{Y}^{l \times n},$ privacy params: $\epsilon, \delta$}
\label{alg:reusecov}
\begin{algorithmic}[1]
\State Draw $E_1 \sim N_{d(d+1)/2}(0, \sigma_1^2),$ where $\sigma_1 =  \frac{1}{n}c(\epsilon, \delta)\normX^2$
\State Compute $\hat{I} = (\frac{1}{n}X^{T}X + E_1 + \lambda I)$
\State Draw $\hat{g} = [\hat{g}_1, \ldots \hat{g}_l] \sim \M(\epsilon/2, \delta/2, X, Y)$ \Comment{$\mathcal{M} = \algotext{GaussMech}$ or Algorithm~\ref{alg:proj}}
\For $i = 1 \ldots l$ 
    \State Set $\tilde{w_i} = \hat{I}^{-1}\hat{g}_i$ \Comment{In practice computed via $\hat{I} = \algotext{QR}$ (Algorithm~\ref{alg:reusecov_emp})}
\EndFor  
\State Return $\tilde{W} = [\tilde{w_1}, \ldots \tilde{w_{l}}]$
\end{algorithmic}
\end{algorithm}

\begin{restatable}{theorem}{sspaccfull}
\label{thm:ssp-acc-full}
With $\M = \algotext{GaussMech}(\epsilon/2, \delta/2, \Delta =  \frac{1}{n}\sqrt{l}\normX \normY)$, Algorithm~\ref{alg:reusecov} is an $(\alpha, \rho, \epsilon, \delta)$ solution to the PRIMO problem with 

\begin{equation}
\alpha = \tilde{O}\left(\norm{\hat{w}}\sqrt{\frac{d\normX^4 \norm{\hat{w}}^2}{n^2} + \frac{ld\normY^2 \normX^2}{n^2}} \right),
\end{equation}
where $\norm{\hat{w}}^2 = \frac{1}{l}\normreg{W^{*}}_F^2$, and $\tilde{O}$ omits terms polynomial in $\frac{1}{\epsilon}, \log(1/\delta), \log(1/\rho)$.
\end{restatable}
\begin{proof} 
The privacy proof follows from a straightforward application of the Gaussian mechanism. We note that releasing each $\hat{g}_i$ privately, is equivalent to computing $X^{T}Y + E_2$, where $E_2 \sim N_{d \times l}(0, \sigma_2^2)$. Now it is easy to compute $l$-sensitivity $\Delta(f)$ of $f(X) = X^{T}Y$. Fix an individual $i$, and an adjacent dataset $X' = X / \{x_i, y_i\} \cup \{x_i', y_i'\}$. Then $f(X)-f(X') = \Delta_V = [y_{i1}x_i - y_{i1}'x_{i1}', \ldots y_{il}x_i - y_{il}'x_{il}']$. Then: 
$$
\norm{f(X) - f(X')} = \sqrt{||\Delta_V||^2_F} = \sqrt{\sum_{j = 1}^{l}\norm{y_{ij}x_i - y_{ij}'x_{ij}'}^2 }\leq \sqrt{l \cdot 4 \norm{\mathcal{X}}^2\norm{\mathcal{Y}}^2} = 2 \sqrt{l} \norm{\mathcal{X}}\norm{\mathcal{Y}}
$$
Hence setting $\sigma_2 = c(\epsilon, \delta) 2\sqrt{l}\norm{\mathcal{X}}\norm{\mathcal{Y}}/\epsilon$ by the Gaussian mechanism \cite{privacybook} publishing $\hat{g}$ satisfies $(\epsilon/2, \delta/2)-DP$. Similarly if $g = X^{T}X$, $\Delta(g) \leq \norm{X}^2$, and so setting $\sigma_1 = c(\epsilon, \delta)\norm{X}^2/\epsilon$, means publishing $\hat{I}$ is $(\epsilon/2, \delta/2)$-DP. By basic composition for DP, the entire mechanism is $(\epsilon, \delta)$-DP.

To prove the accuracy bound we follow the general proof technique developed in \cite{revisit} analysing the accuracy guarantees of the ridge regression variant of $\algotext{SSP}$ in the case $\lambda_{\min}(X^{T}X) = 0$, adding some mathematical detail to their exposition, and doing the appropriate book-keeping to handle our setting where the privacy level (as a function of the noise level) guaranteed by $E_1$ and $E_2$ differ. The reader less interested in these details can skip to Equation~\ref{eq:final} below for the punchline. 

Fix a specific index $i \in [l]$, and let $y = y_i$. We will analyze the prediction error of $\tilde{w_i}$ e.g. $f(\tilde{w}_i) - f(w_i^*)$. Then the following result is stated in \cite{revisit} for which provide a short proof:

\begin{lemma}[\cite{revisit}]\label{lem:prederror}
$$
f(\hatw) - f(\wstar) = \frac{1}{n}(||y-X\hatw||^2 - ||y-X\wstar||) = \frac{1}{n}||\hatw - \wstar||^2_{X^{T}X}
$$
\end{lemma} 

\begin{proof}
We note that all derivatives of orders higher than $2$ of $f(w) = \frac{1}{n}||y-Xw||^2$ are zero, and that $\nabla f_{\wstar} = 0$ by the optimality of $\wstar$. We also note that the Hessian $\nabla^2 f_w = \frac{1}{n}X^{T}X$ at all points $w$. Then by the Taylor expansion of $f(w)$ around $\wstar$:
$$ f(\hatw) = f(\wstar) + (\hatw - \wstar)\cdot \nabla f_{\wstar} + \frac{1}{n}(\hatw - \wstar)'X^{T}X(\hatw - \wstar)
$$
Which using $\nabla f_{\wstar} = 0$ and rearranging terms gives the result. 
\end{proof}

Now Corollary 7 in the Appendix of \cite{revisit} states (without proof) the below identity, which we provide a proof of for completeness in  Lemma~\ref{lem:sherman} in the Appendix: 
\begin{equation}
\label{eq:decomp_e2}
    \hatw - \wstar = (-X^{T}X + \lambda I + E_1)^{-1}E_1 \wstar - \lambda (\covx + \lambda I + E_1)^{-1}\wstar + (\covx + \lambda I  + E_1)^{-1}E_2
\end{equation}

Hence, still following \cite{revisit}, for any psd matrix $A$, $||\hatw - \wstar||_A^2 \leq$
\begin{equation}\label{matrixnorm}
 3 ||(\covx + \lambda I + E_1)^{-1}E_1 \wstar||_A^2 + 3\lambda^2 \normreg{(\covx + \lambda I + E_1)^{-1}\wstar}_A^2 + 3||(\covx + \lambda I + E_1)^{-1}E_2||_A^2
\end{equation}

\begin{lemma}[\cite{revisit}]\label{matrixneq}
With probability $1-\rho$, $||E_1||_2 \leq (\lambda_{\min}(\covx) +  \lambda) / 2$, and hence
$$\covx + \lambda I + E_1 	\succ .5(\covx + \lambda I)$$ 
\end{lemma}

We also remark that $||By||_A^2 = (By)^{T}ABy = ||y||_{B^{T}AB}^2$ for any vector $y$, and matrices $A, B$.

Hence, Inequality~\ref{matrixnorm}, with $A = \covx$ can be further simplified to:
\begin{multline}\label{normexp}
  3 ||(\covx + \lambda I + E_1)^{-1}E_1 \wstar||_A^2 + 3\lambda^2 \normreg{(\covx + \lambda I + E_1)^{-1}\wstar}_A^2 + 3||(\covx + \lambda I + E_1)^{-1}E_2||_A^2 \leq \\
  O\left(\normreg{E_1 \wstar}_{(\covx + \lambda I + E_1)^{-1}}^2 + \lambda^2 \normreg{\wstar}_{(\covx + \lambda I + E_1)^{-1}}^2 + \normreg{E_2}^2_{(\covx + \lambda I + E_1)^{-1}}\right) \leq \\
  \text{(Lemma~\ref{matrixneq})} \; \; O\left(\normreg{E_1 \wstar}_{(\covx + \lambda I)^{-1}}^2 + \lambda^2 \normreg{\wstar}_{(\covx + \lambda I)^{-1}}^2 + \normreg{E_2}^2_{(\covx + \lambda I)^{-1}}\right)
\end{multline}
By basic properties of the trace we have: $\text{tr}(({\lambda I + \covx)^{-1}}) \leq d \lambda_{max}({\lambda I + \covx^{-1}}) = \frac{d}{\lambda_{\min}(\lambda I + \covx)} = \frac{d}{(\lambda_{\min} + \lambda)}$, and $\normreg{\wstar}_{(\covx + \lambda I)^{-1}}^2 \leq \frac{\norm{\wstar}^2}{\lambda}$. Continuing from \cite{revisit} by their Lemma 6, we can bound each $\normreg{E_1 \wstar}_{(\covx + \lambda I)^{-1}}^2$ and $\normreg{E_2}^2_{(\covx + \lambda I)^{-1}}$.

\begin{lemma}[\cite{revisit}]\label{lem:jl}
Let $w \in \mathbb{R}^d$ and let $E$ a symmetric Gaussian matrix where the upper triangular region is sampled from $N(0, \sigma^2)$ and let $A$ be any psd matrix. Then with probability $1-\rho$:

$$
||Ew||_A^2 \leq \sigma^2\text{tr}(A)\norm{w}^2 \log(2d^2/\rho)
$$
\end{lemma} 

 Then recalling that: 

\begin{itemize}
    \item $\sigma_1^2 = \tilde{O}(\normX^4/\epsilon^2)$
    \item $\sigma_2^2 = \tilde{O}(l\normX^2 \normY^2/\epsilon^2)$
\end{itemize}
 Plugging into Lemma~\ref{lem:jl}, and bringing it all together we get: 
 
\begin{multline} \label{eq:penult}
n \cdot |f(\hatw)-f(\wstar)| \leq \normreg{\wstar-\hatw}_{\covx} = \\
O\left(\normreg{E_1 \wstar}_{(\covx + \lambda I)^{-1}}^2 + \lambda^2 \normreg{\wstar}_{(\covx + \lambda I)^{-1}}^2 + \normreg{E_2}^2_{(\covx + \lambda I)^{-1}}\right) = \\
\tilde{O}\left( \frac{d}{\lambda_{\min} + \lambda} \norm{\wstar}^2 (\normX^4/\epsilon^2)\log(2d^2/\rho) + 
\lambda \norm{\wstar}^2 + \frac{d}{\lambda_{\min} + \lambda}l\normX^2 \normY^2/\epsilon^2\log(2d^2/\rho)\right)
\end{multline}

Upper bounding Equation~\ref{eq:penult} by taking $\lambda_{\min} = 0$, we minimize over $\lambda$ setting 

\begin{multline}\label{eq:final}
\lambda = \tilde{O}\left(
\frac{1}{\epsilon}\sqrt{d \log(2d^2/\rho)}\norm{X} \sqrt{\normX^2 + \frac{l  \normY^2}{\norm{\wstar}^2}} \right) \implies \\
|f(\hatw)-f(\wstar)| = \tilde{O}\left(
\frac{1}{n\epsilon}\sqrt{d \log(2d^2/\rho)}\norm{X} \sqrt{\normX^2 \norm{\wstar}^4 + l  \normY^2\norm{\wstar}^2} \right)
\end{multline}

Now if we are in the $\eta-$small regime, we have 
$\normY \leq \eta \normX \norm{\wstar}$, and so 
$$ l\normY^2 \wstar^2 \leq l \eta^2 {\normX}^2 {\norm{\wstar}}^4,$$ 
which reduces Equation~\ref{eq:final} to: 
$$
|f(\hatw) - f(\wstar)| = \tilde{O}\left(
\frac{1}{n\epsilon}\sqrt{d \log(2d^2/\rho)}\norm{X}^2\norm{\wstar}^2(\eta\sqrt{l}) \right),
$$ as desired. 

\end{proof}

Inspecting Theorem~\ref{thm:ssp-acc-full}, we see that when $l < \min(\frac{n}{\sqrt{d}}, \frac{\normX^2 \normW^2}{\normY^2})$, Algorithm~\ref{alg:reusecov} with $\mathcal{M}$ the Gaussian mechanism achieves error $\tilde{O}\left(\frac{\sqrt{d}\normX^2 \norm{\hat{w}}^2}{n}\right)$. Since this matches the lower bound in Equation~\ref{eq:lb} for a single private regression in this case we say that we have achieved ``PRIMO for Free!''

\section{Improved Algorithms for Large $l$}
\label{sec:query_release}
The improvement of $\algotext{ReuseCovGauss}$ over the naive PRIMO baseline in the previous section only applies in the parameter regime where the asymptotic error in Equation~\ref{eq:ssp_error_terms} is dominated by the covariance error term. In the regime where the association error term dominates, \algotext{ReuseCovGauss} still incurs a $\sqrt{l}$ multiplicative factor in the error term, which does not improve over the baseline. In this section we show that if we are willing to relax from full DP to either feature DP or label DP, we can improve the dependence on $l$. Specifically, via a reduction from privately computing the association term to computing the image of a private vector under a public matrix,  under feature DP we can obtain improved bounds for PRIMO when $l > \frac{n}{\sqrt{d}}$ (Subsection~\ref{subsec:feat}), and under \emph{label differential privacy} we obtain improved bounds when $l > n^{2/3}$ (Subsection~\ref{subsec:label}). In these cases, we obtain the surprising result that the MSE $\alpha$ has no explicit asymptotic dependence on $l$.
\subsection{PRIMO Under Feature Differential Privacy}
\label{subsec:feat}

\begin{algorithm}[!ht]{$\algotext{GaussProj}_{Y}$}
\caption{Input:  $X \in \mathcal{X}^n \subset (\mathbb{R}^d)^n$, $Y = [y_1, \ldots y_{l}] \in \mathcal{Y}^{l \times n},$ privacy params: $\epsilon, \delta$.}
\label{alg:proj}
\begin{algorithmic}[1]
\State Let $r = c(\epsilon, \delta)\frac{1}{n}\sup_i||y^i||_2 \normX$ 
\State Sample $w \sim N(0, 1)^{dl}$
\State Let $\tilde{g} = g + rw$, where $g = \frac{1}{n}X^{T}Y$ \Comment{Up to this point this is just $\algotext{GaussMech}$}
\State Formulate $C = C(Y) \in \frac{1}{n}\mathcal{Y}^{dl \times dn}, \vecx$. \Comment{$C$ is only used in the projection step}
\State Let $\hat{g} = \text{argmin}_{g \in K}\norm{g-\tilde{g}}^2$, where $K = C(\sqrt{n}\normX B_1)$.
\end{algorithmic}
\end{algorithm}

Let us reconsider the problem of privately computing the association term $\frac{1}{n}X^{T}Y$ where $Y \in \mathcal{Y}^{n \times l}$, $X \in \mathcal{X}^{n} \subset \mathbb{R}^{n \times d}$, and $Y$ is considered public, so $\mathcal{M}$ is only constrained to be feature DP. Let $\vecx = (x_{11}, \ldots, x_{1n}, x_{21}, \ldots x_{d1}, \ldots x_{dn}) \in \mathbb{R}^{nd}$. Then since $\frac{1}{n}X^{T}Y$ is \emph{linear} in every entry of $X$, there exists a matrix $C(Y) \in \frac{1}{n}\mathcal{Y}^{dl \times dn}$ such that 
$\frac{1}{n}X^{T}Y = C\cdot \vecx$. We explicitly derive an expression for $C$ in Subsection~\ref{subsec:app_query} of the Appendix. Note that $\norm{\vecx}^2 \leq n\normX^2$, hence $\frac{1}{n}X^{T}Y \in C(\sqrt{n}\normX B_1)$, where $B_1 = \{x \in \mathbb{R}^{nd}: \norm{x} \leq 1\}$. 

Algorithm~\ref{alg:proj} is our projection-based subroutine for privately computing $\frac{1}{n}X^{T}Y$, which we call $\algotext{GaussProj}_{Y}$. We first state the squared error of $\algotext{GaussProj}_Y$ in Theorem~\ref{thm:proj-acc}, and then translate this error into the PRIMO error of $\algotext{ReuseCovProj}_Y$ (Algorithm~\ref{alg:reusecov} with $\mathcal{M}$ = Algorithm~\ref{alg:proj}) in Theorem~\ref{thm:ssp-query-acc}. The crux of the proof is Lemma~\ref{lem:proj_lemma} from \cite{proj}, which uses a geometric argument to bound the error after the projection.

\begin{restatable}{theorem}{projacc}
\label{thm:proj-acc}
Let $\hat{g}, g$ as in Line $4$ of Algorithm~\ref{alg:proj}. Then with probability $1-\rho:$
$$
\norm{g-\hat{g}}^2 = O\left(\frac{l\sqrt{d}c(\epsilon, \delta)\sqrt{\log(2/\rho)}\normY^2 \normX^2}{n}\right)
$$
\end{restatable}
\begin{proof}
$\mathcal{M}$ is $(\epsilon, \delta)$ differentially private by the Gaussian Mechanism and post-processing \cite{privacybook}. So we focus on the high probability accuracy bound. By Lemma 4 from \cite{proj} we have, letting $K, r, w$ as defined in Algorithm~\ref{alg:proj}, that: 
\begin{equation}
\norm{\hat{g}-g}^2 \leq 4 ||rw||_{K^{\circ}} = r\sup_{x \in K}x \cdot w
\end{equation}
Using the fact that the $l_2$ norm is self-dual, we have: 
\begin{equation}
    \norm{\hat{g}-g}^2 \leq 4r ||w||_{K^{\circ}} \leq 4r\normX\sqrt{n}\sup_{z \in B_{1}}(Cz)\cdot w = 4r\normX\sqrt{n}\norm{C^{T}w}
\end{equation}

So in order to bound $\norm{\hat{g}-g}^2$ with high probability it suffices to bound $\norm{C^{T}w}$ with high probability. This is the content of the Hanson-Wright Inequality for anisotropic random variables \cite{vershynin}. 
\begin{lemma}[\cite{vershynin}]\label{hanson}
Let $C^{T}$ an $m \times n$ matrix, and $X \sim N(0,1)^{n} \in \mathbb{R}^n$. Then for a fixed constant $c > 0$:
$$
\mathbb{P}[\;|\; \normreg{C^{T}X}-\normreg{C}_F\;| > t] \leq 2\text{exp}({\frac{-ct^2}{\normreg{C}_F}})
$$
\end{lemma}

Lemma~\ref{hanson} shows that $\norm{C^{T}w} = O(\normreg{C}_F\sqrt{\log(2/\rho)})$ with probability $1-\rho$. Given $x_{kj} \in X$, let $m = (k-1)n + j$ be the corresponding column of $C$, and let $c^{m}$ denote this column. Then $||c^{m}||_2 = \frac{1}{n}||(y_{1j}, \ldots y_{lj})||_2 = \frac{1}{n}||y^{j}||_2 \leq \frac{\sqrt{l}}{n}$. Hence $\normreg{C}_F = \frac{\sqrt{d \sum_{i=1}^{n}\norm{y^i}^2}}{n}$.

 Plugging in the value of $r$ gives, with probability $1-\rho$:
\begin{align}
    \frac{1}{dl}\mathbb{E}[\norm{\hat{g}-g}^2] &= O\left(\frac{1}{dl} \cdot \sqrt{n}\normX\frac{\sqrt{d \sum_{i=1}^{n}\norm{y^i}^2}}{n}\sqrt{\log(2/\rho)} \cdot \frac{c(\epsilon, \delta)\normX \sup_{i}\norm{y^i}}{n}\right) \\
    &=   O\left(\frac{c(\epsilon, \delta)\sqrt{\log(2/\rho)} \sqrt{\frac{1}{n}\sum_{i = 1}^{n}\norm{y^i}^2}\sup_{i}\norm{y^i} \normX^2}{nl\sqrt{d}}\right), 
\end{align}
Since both terms involving $Y$ in the numerator are $\leq \sqrt{l}\normY$ the bound follows. We note that since $Y$ is public, in practice we can compute these terms rather than using $\normY$, the worst case bound.
\end{proof}
Notice that the mean squared error of the Gaussian mechanism without the projection is $O({r^2}) = O(\frac{l\normX^2\normY^2}{n^2})$, which for $l\sqrt{d} \gg n $ is strictly larger than the error of the projection mechanism. We also note that the bound in Theorem~\ref{thm:proj-acc} is strictly better than the error given by applying the Median Mechanism algorithm of \cite{median_mech} for low-sensitivity queries, which is tailored for $l_\infty$ error, and which also requires discrete $\mathcal{X}$. For example, when $\mathcal{Y} = \{0,1\}, \mathcal{X} = \{0,1\}^d$, then $\normY = 1, \normX = \sqrt{d}$, and so Theorem~\ref{thm:proj-acc} gives a bound of $\tilde{O}(\frac{\sqrt{d}}{n})$, whereas the Median Mechanism gives $\tilde{O}(\frac{d^{2/3}\log(dl)^2}{n^{2/3}})$ for the mean squared error. We now state the accuracy guarantees of $\algotext{ReuseCovProj}_Y$. The proof follows the proof of Theorem~\ref{thm:ssp-acc-full}, substituting the bound from Theorem~\ref{thm:proj-acc} into the $\norm{E_2}^2$ term in Equation~\ref{eq:decomp_e2}. We defer the full proof to the Appendix.  

\begin{restatable}{theorem}{sspqueryacc}
\label{thm:ssp-query-acc}
Let $\mathcal{A}$ denote the variant of Algorithm~\ref{alg:reusecov} with $\M = \algotext{GaussProj}_{Y}(X, Y, \epsilon/2, \delta/2)$. Then $\mathcal{A}$ is an $(\alpha, \rho, \epsilon, \delta)$ solution to the feature DP PRIMO problem with 
$$
\alpha =  \tilde{O}\left(\norm{\hat{w}} \sqrt{\frac{d \norm{\hat{w}}^2 (\normX^4)}{n^2} + \frac{\sqrt{d}\normY^2 \normX^2}{n}}\right),
$$
where $\tilde{O}$ omits terms polynomial in $\frac{1}{\epsilon}, \log(1/\delta), \log(1/\rho)$.
\end{restatable}

Inspecting Theorem~\ref{thm:ssp-query-acc} in the $l > \frac{n}{\sqrt{d}}$ regime where the projection improves the error over the Gaussian Mechanism, when $n < \frac{\sqrt{d}\normX^2 \normW^2}{\normY^2}$ is not too large, the dominant term in the error is $\tilde{O}\left(\frac{\sqrt{d}\normX^2 \norm{\hat{w}}^2}{n} \right), $ which matches the lower bound~\ref{eq:lb}, and so we achieve PRIMO for Free! When $n$ is sufficiently large the dominant error term is $\tilde{O}\left(\frac{\norm{\hat{w}}{d}^{1/4}\normY \normX}{\sqrt{n}} \right)$ which is a factor of $\frac{n}{d^{1/4}}$ worse than the lower bound.

\subsection{PRIMO Under Label Differential Privacy}
\label{subsec:label}
The setting where $X$ is public and $Y$ is considered private is strictly easier to satisfy than setting in which they are both considered private, and in particular in order for Algorithm~\ref{alg:reusecov} to satisfy label DP, we only need to add noise to terms that involve $Y$. Thus we can compute $\hat{I}$ in Line~2 of Algorithm~\ref{alg:reusecov} exactly without adding any noise $E_1$. We still need to compute the association term $\frac{1}{n}X^{T}Y$ privately in $Y$, where now the privacy guarantee holds over rows of $Y$ rather than rows of $X$. Computing $\frac{1}{n}X^{T}Y$ privately is equivalent to computing  $\frac{1}{n}Y_{l \times n}^{T}X_{n \times d}$, where we've added subscripts indicating the dimensions of the matrices. Then by direct analogy to the previous section, we can compute $\frac{1}{n}X^{T}Y$ via Algorithm $2$, where Theorem~\ref{thm:proj-acc} holds by switching $X$ and $Y$, and $d, l$. We call this variant of Algorithm~\ref{alg:proj} $\algotext{GaussProj}_X$. Thus invoking Theorem~\ref{thm:proj-acc} the error in computing the association term is with probability $1- \rho$:
\begin{equation}
\label{eq:proj_error_label}
    ||g-\hat{g}||_2^2 = O\left(\frac{c(\epsilon, \delta)\sqrt{\log(2/\rho)}|\mathcal{X}|^2\norm{\mathcal{Y}^l}^2 d \sqrt{l}}{n}\right)
\end{equation}
The error in Equation~\ref{eq:proj_error_label} improves over the error given by $\mathcal{M} = \algotext{GaussMech}$ when $O(\frac{d \sqrt{l}}{n}) < O(\frac{dl^2}{n^2}) \implies l > n^{\frac{2}{3}}$, which we summarise in row $5$ of Figure~\ref{tbl:summary}. Interestingly, in contrast to the public label setting, this condition has no dependence on the dimension $d$. Substituting Equation~\ref{eq:proj_error_label} into Equation~\ref{eq:decomp_e2} with $E_1 = 0$ since we are in the label-private setting and don't have to add noise to the covariance matrix, we can optimize over $\lambda$ to get:

\begin{theorem}
\label{thm:ssp-query-acc-label}
Let $\mathcal{A}$ denote the label-private variant of Algorithm~\ref{alg:reusecov} with $E_1 = 0, \M = \algotext{GaussProj}_{X}(X, Y, \epsilon/2, \delta/2)$. Then $\mathcal{A}$ is an $(\alpha, \rho, \epsilon, \delta)$ solution to the label DP PRIMO problem with 
$$ \alpha = \tilde{O}\left(\norm{\hat{w}}\frac{|\mathcal{X}|\norm{\mathcal{Y}^l}\sqrt{d}}{nl^{1/4}}\right) $$
\end{theorem}
We remark that while the form of Theorem~\ref{thm:ssp-query-acc-label} appears to suggest that the MSE decreases with increasing $l$, which is counterintuitive, $l$ actually appears in the numerator through $||\mathcal{Y}^l||_2 = O(|\mathcal{Y}|\sqrt{l})$, giving MSE with (mild) dependence $O(l^{1/4})$ on $l$.

\subsection{Computational Efficiency}
\label{subsec:complexity_sketch}

In this section we discuss the computational complexity of Algorithm~\ref{alg:reusecov} which can be broken down into $3$ components: 
\begin{itemize}
\item Step $1$: Forming $\covx$
\item Step $2$: In the case where $\mathcal{M}$ is the projection algorithm, Line $5$ in Algorithm~\ref{alg:proj}: 

$$\hat{g} = [\hat{g}_1, \ldots \hat{g}_l] = \text{argmin}_{g \in C(\sqrt{n}\normX B_1)}\normreg{g-\tilde{g}}_2^2$$

\item Step $3$: the cost of computing $\hat{I}^{-1}\hat{g}_i = (\covx + \lambda I + E_1)^{-1}\hat{g}_i \; \forall i \in [l]$
\end{itemize}

Step $1$ forming the covariance matrix $X^{T}X$ is a matrix multiplication of two $d \times n$ matrices, which can be done via the naive matrix multiplication in time $O(nd^2)$, and via a long-line of ``fast'' matrix multiplication algorithms in time $O(d^{2 + \alpha(n)})$; For example if $n < d^{.3}$ it can be done in time that is essentially $O(d^2)$ \cite{fast_mat}. Step $2$ corresponds to minimizing a quadratic over a sphere. Setting $A = C^TC \in \mathbb{R}^{dn \times dn}, b = 2C^{T}\tilde{g} \in \mathbb{R}^{dl}$, then $\hat{g} = Cx$, where $x \in \mathbb{R}^{nd}$ is the minimizer of:
\begin{align}
    & \min_{x \in B_1}x^{t}Ax - b^{t}x  \\
    & \text{ s.t. } \normreg{x}_2 \leq \sqrt{n}\normX
\end{align}
Now, given the spectral decomposition of $A = U\Lambda U^{T}$, and the coordinates of $b$ in the eigenbasis $U^{T}b$, Lemma $2.2$ in \cite{hager} gives a simple closed form for $x$ that computes each coordinate in constant time. Since there are $nd$ coordinates of $x$, this incurs an additional additive factor of $O(nd)$ in the complexity, which is dominated by the cost of diagonalizing $A$. So the complexity of this step is the complexity of diagonalizing $A = C^{T}C$, or equivalently finding the right singular vectors of $C$, plus the complexity of computing $U^{T}b$.
This is seemingly bad news, as $C \in \mathbb{R}^{dn \times dl} $ is a very high-dimensional matrix, and the complexity for computing the SVD of $C$ without any assumptions about its structure is $O(d^3ln \min(l, n))$ \cite{GoluVanl96}. However, it is evident from the construction of $C$ in Subsection~\ref{sec:query_release} that $C = I_d \otimes \frac{1}{n}Y^{T},$ where $\otimes$ is the Kronecker product. In Appendix~\ref{sec:complex} we show that using properties of the Kronecker product, we can compute the SVD$(C)$ in the same time as computing SVD$(Y)$, or $O(nl \min(n,l))$. $U^Tb$ can be computed in time $O(nld)$ using similar tricks based on the Kronecker decomposition of $C$. 

Finally, Step $3$ can be completed by solving the equation $\hat{I}\hat{w}_i = \hat{g}_i, i = 1 \ldots l$ via the conjugate gradient method, which takes time $O(\gamma(\hat{I})d^2\log(1/\epsilon))$ to compute an $\epsilon$-approximate solution \cite{mahoney} where $\gamma(\hat{I})$ is the condition number. We note that this has to be done separately for each $i = 1 \ldots l$ giving total time $O(l \cdot \gamma(\hat{I})d^2\log(1/\epsilon))$. Alternatively, an exact solution $\hat{I}^{-1}\hat{g}_i$ can be computed directly using the QR decomposition of the matrix $\hat{I}$. The decomposition $\hat{I} = QR$ can be computed in time $O(d^3)$ \cite{mahoney} and does not depend on the $\hat{g}_i$, after which using $R\hat{w}_i = Q^{T}\hat{g}_i$, $\hat{w}_i$ can be computed in time $O(d^2)$ via backward substitution. This gives a total time complexity of $O(d^3 + ld^2)$. So if $d$ and $\frac{1}{\gamma(\hat{I})}$ are sufficiently small relative to $l$, e.g. if $l = \tilde{\Omega}(d/(\gamma(\hat{I})))$, it will be faster to use the $QR$ decomposition based method. 

Putting this all together we have:
\begin{theorem}
\label{thm:complex}
The complexity of Algorithm~\ref{alg:reusecov_emp} is $O(\max(\min(nl^2, n^2l), nld, nd^2, ld^2, d^3))$.
\end{theorem}

We summarize these algorithmic changes to \algotext{ReuseCovProj} in Algorithm~\ref{alg:reusecov_emp}. Theorem~\ref{thm:complex} shows that when $n > d > l$, the complexity of Algorithm~\ref{alg:reusecov_emp} is $O(nd^2)$ or the cost of forming the covariance matrix. This motivates the results in Subsection~\ref{sec:sub_ssp} of the Appendix where we show how sub-sampling $s < n$ points and using matrix Chernoff bounds can improve this to $O(sd^2)$.



\section{Experiments}\label{sec:experiments}
\textbf{Datasets.} We evaluate the accuracy of our algorithms $\algotext{ReuseCovGauss}$ and $\algotext{ReuseCovProj}_X$ on the task of genomic risk prediction using real $X$ data from two of the largest publicly available databases, and simulated outcomes $Y$ so that we can easily vary the number of outcomes $l$. In addition, in Appendix~\ref{sec:app_exp} we include experiments on two additional datasets over smaller values of $l$, one constructed by sub-sampling MNIST \cite{mnist} and generating synthetic outcomes from a noisy linear model, and the other from entirely synthetic Gaussian data with outcomes generated either from a $2$-layer MLP or a noisy linear model. 

The genomic datasets are from two sources: the 1000 Genomes project (1KG) \cite{1KG}, and the Database of Genomes and Phenotypes \cite{dbgap} (accession phs000688.v1.p1). The datasets contain $n = 5008, 6042$ haplotypes at $d = 78961, 16721$ SNPs, respectively. Whenever experiments are run with a fixed value of $d < 78961, 16721$, we have randomly sub-sampled $d$ SNPs without replacement. In order to vary the number of phenotypes $l$ we generate synthetic phenotype data using our haplotype dataset. After centering our haplotype matrix $X$ by subtracting off the row means we generate synthetic phenotypes $y_i$ for $i = 1\ldots l$ by generating a random $\theta_{i} \sim \mathcal{N}(0, \frac{I_d}{\sqrt{d}}), y_{ij} \sim \theta_{i} \cdot x_j + \mathcal{N}(0, 1).$ In order to evaluate our algorithms at larger values of $n$ than exist in our genomic datasets, we train a GAN on the dbGaP dataset following \cite{lukman}, and use it to generate synthetic samples to simulate having up to a million data points. \\

\textbf{Experimental Details.} For a given dataset and setting of $(n, d, l)$ we run $\algotext{ReuseCovProj}_X$ and $\algotext{ReuseCovGauss}$ $10$ different times, and calculate the resulting average MSE. We study two different regimes, one where $n$ is of moderate size corresponding to real genomic datasets, $d$ is small, and $l$ is varied from $1$ to $1e5$ (Figure~\ref{fig:graphs}). The other regime models the scenario where in the near future genomic databases will have millions of individuals (e.g. UK Biobank already has $> 400000$ samples \cite{biobank}), generating $n = 1e6$ samples from the GAN trained on dbGaP and again varying $l$ from $1$ to $1e5$ (Figure~\ref{fig:large_synth}). Note that we do not have experiments with larger values of $d$ for computational reasons; as discussed in Section~\ref{sec:complex} the runtime of both algorithms is quadratic in $d$.

Rather than plotting the MSE $\alpha$ directly, in Figure~\ref{fig:graphs} we plot the coefficient of determination $R^2$, which is equivalent in that $R^2 = 1-\frac{\alpha}{\frac{1}{nl}SS_{tot}}$. We use $R^2$ because it is easier to interpret given that achieving error better than the constant predictor implies $R^2 \in (0,1]$. In Figures~\ref{fig:logerr_dbgap},\ref{fig:large_synth} we also visualize the results using a different metric, the log of the ratio of the MSE of the private predictor to the MSE of the optimal non-private predictor, where a larger ratio indicates a greater multiplicative error. 

\textbf{Realistic $n$, varying $l$}. 
We find that for values of $l$ starting as small as $11$ for 1KG and $101$ for dbGaP, $\algotext{ReuseCovProj}_X$ achieves higher $R^2$ than $\algotext{ReuseCovGauss}$, and is able to achieve non-trival $R^2$ for large $l$, which is consistent with our Theorem~\ref{thm:ssp-query-acc}. For $l < 11$ on 1KG and $l < 101$ on dbGaP, $\algotext{ReuseCovGauss}$ achieves higher $R^2$. In addition to being consistent with the theory, this is intuitive; when $l$ is small the bias induced by the projection outweighs the reduction in the overall noise. Figures~\ref{fig:r2_1kg}, \ref{fig:r2_all_dbgap} show that for $l > 101, 801$ on 1KG and dbGaP respectively $\algotext{ReuseCovGauss}$ has negative $R^2$ -- by contrast, zooming in on the $R^2$ curve for $\algotext{ReuseCovProj}$ in Figures~\ref{fig:r2_proj_1kg}, \ref{fig:r2_proj_dbgap} we are able to achieve non-trivial $R^2$ at all values of $l$, and see minimal increase in error as $l$ increases. 

\begin{figure}[!h]
\centering
\begin{subfigure}[c]{.45\textwidth}
\includegraphics[width=\textwidth]{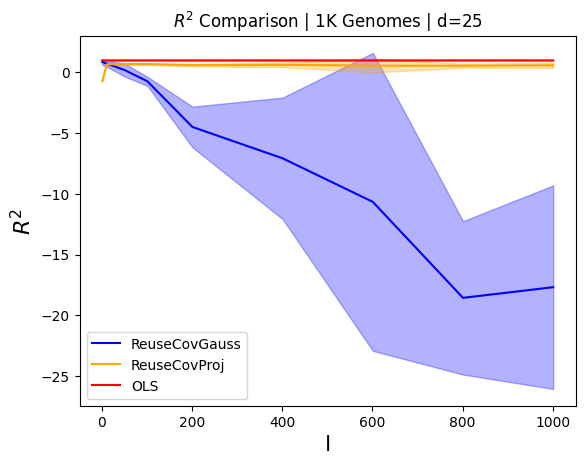}
\caption{}
\label{fig:r2_1kg}
\end{subfigure}
\hspace{.02\textwidth}
\begin{subfigure}[c]{.45\textwidth}
\includegraphics[width=\textwidth]{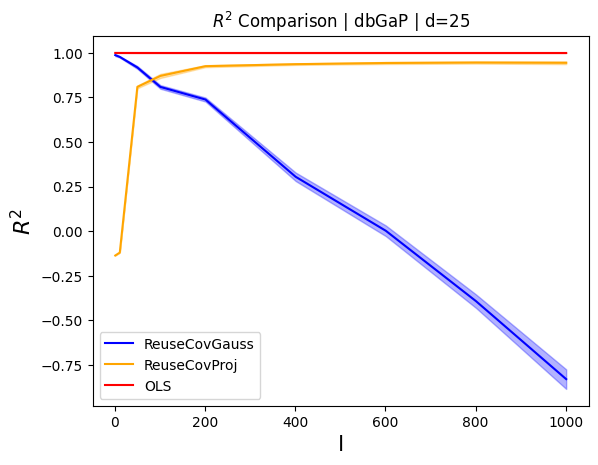}
\caption{}
\label{fig:r2_all_dbgap}
\end{subfigure}
\vspace{.02\textwidth}
\begin{subfigure}[c]{.45\textwidth}
\includegraphics[width=\textwidth]{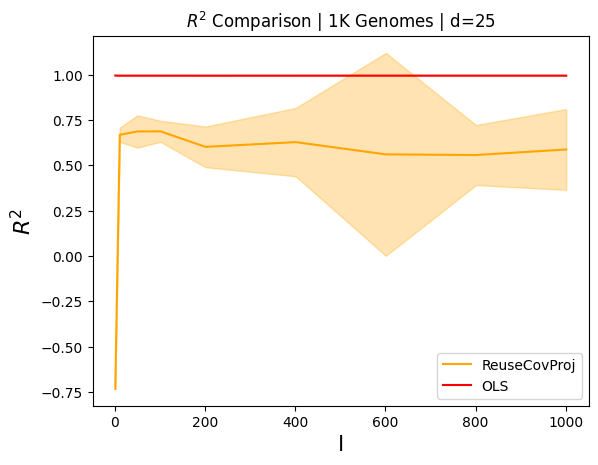}
\caption{}
\label{fig:r2_proj_1kg}
\end{subfigure}
\hspace{.02\textwidth} 
\begin{subfigure}[c]{.45\textwidth}
\includegraphics[width=\textwidth]{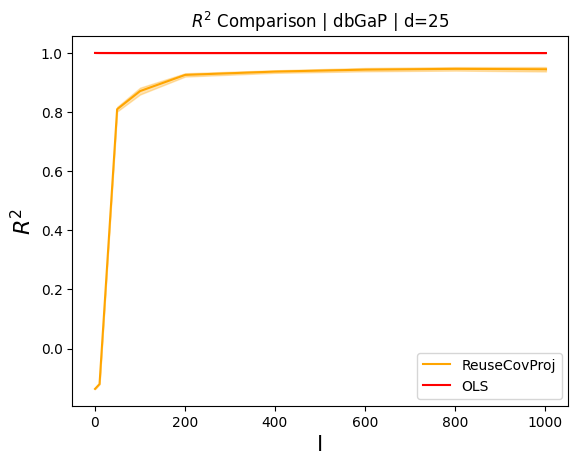}
\caption{}
\label{fig:r2_proj_dbgap}
\end{subfigure}
\begin{subfigure}[c]{.45\textwidth}
\includegraphics[width=\textwidth]{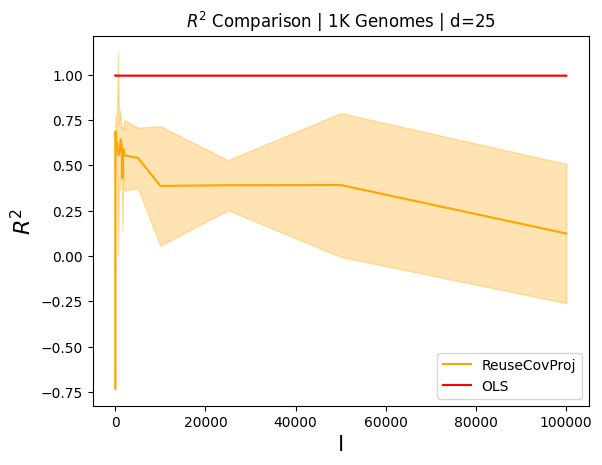}
\caption{}
\label{fig:app_r2_1kg}
\end{subfigure}
\hspace{.02\textwidth}
\begin{subfigure}[c]{.45\textwidth}
\includegraphics[width=\textwidth]{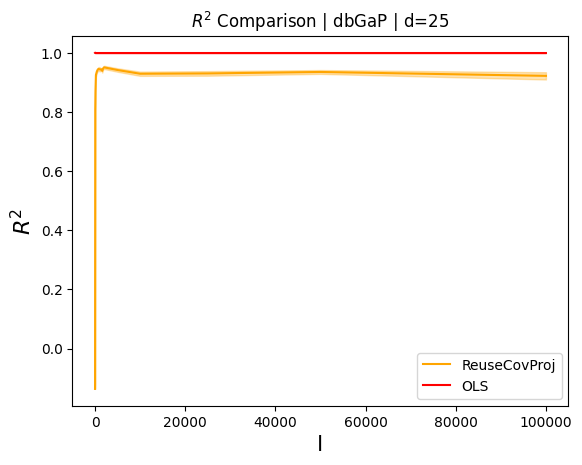}
\caption{}
\label{fig:app_r2_all_dbgap}
\end{subfigure}
\caption{Each $R^2$ value is averaged over $10$ iterations. The shaded area around the lines indicates the error bars for the $R^2$ value at a given value of $(l, d)$. In Figures (a)-(d) we plot the average $R^2$ for $d = 25, l = (1, 11, 101, 201, 401,601, 801, 1001),$ fixing $(\epsilon, \delta) = (5,\frac{1}{n^2})$ with $n = 5008, 6042$. In Figures (e)-(f) we show $l$ up to $1e5$.}
\label{fig:graphs}
\end{figure}

\textbf{Large $n$, varying $l$.}
Our upper bounds suggest that for large values of $n$, $\algotext{ReuseCovGauss}$ will outperform $\algotext{ReuseCovProj}_X$, and that the MSE should increase polynomially with increasing $l$, or equivalently linearly in $\log l$ when we take the log ratio. Both of these trends can be seen in Figure~\ref{fig:large_synth}(a). Given that the relative performance of $\algotext{ReuseCovProj}_X$ is improving as $l$ increases, we expect that for larger values of $l$ that were computationally prohibitive to run at such large $n$, $\algotext{ReuseCovProj}_X$ would again be the superior algorithm.

\begin{figure}[!h]
\centering
\begin{subfigure}[c]{.45\textwidth}
\includegraphics[width=\textwidth]{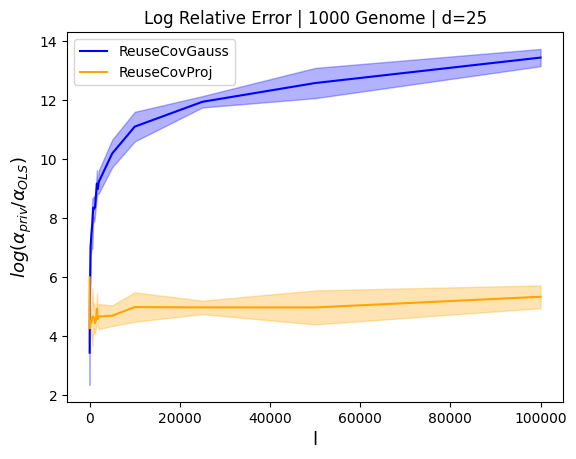}
\caption{}
\label{fig:app_logerr_1kg}
\end{subfigure}
\hspace{.02\textwidth}
\begin{subfigure}[c]{.45\textwidth}
\includegraphics[width=\textwidth]{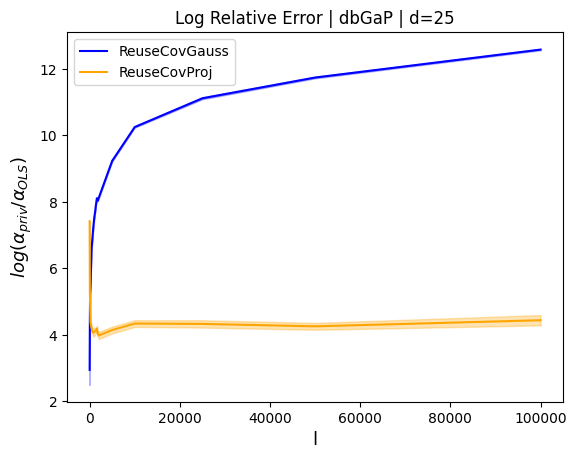}
\caption{}
\label{fig:app_logerr_dbgap}
\end{subfigure}
\vspace{.02\textwidth}
\begin{subfigure}[c]{.45\textwidth}
\includegraphics[width=\textwidth]{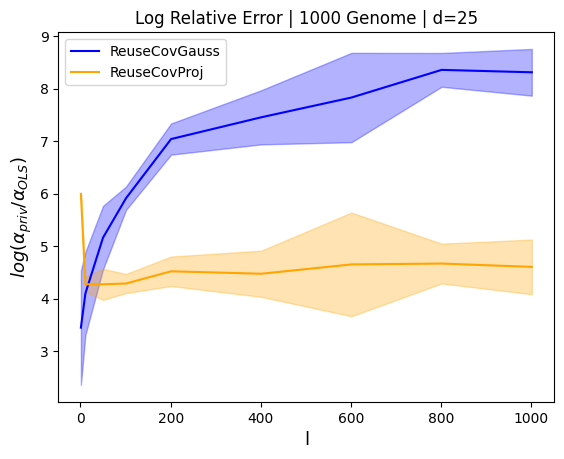}
\caption{}
\label{fig:logerr_1kg}
\end{subfigure}
\hspace{.02\textwidth}
\begin{subfigure}[c]{.45\textwidth}
\includegraphics[width=\textwidth]{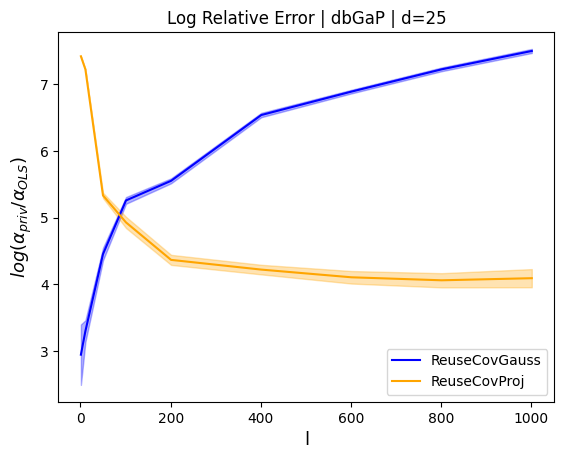}
\caption{}
\label{fig:logerr_dbgap}
\end{subfigure}
\caption{Comparing the log of the ratio of the squared loss of the private estimator to the square loss of the OLS estimator. Each value is averaged over 10 iterations. The shaded area around the lines indicates the error bars at the given value of $l$. (a) and (b) range $l$ up to $100,000$, (c) and (d) show $l = (1, 11, 101, 201, 401,601, 801, 1001)$ while we fixed $d=25$ and $(\epsilon, \delta) = (5,\frac{1}{n^2})$ with $n = 5008, 6042$ for 1000 Genomes and dbGaP respectively}
\label{fig:app_logerr}
\end{figure}

\begin{figure}[!h]
\centering
\begin{subfigure}[c]{.45\textwidth}
\includegraphics[width=\textwidth]{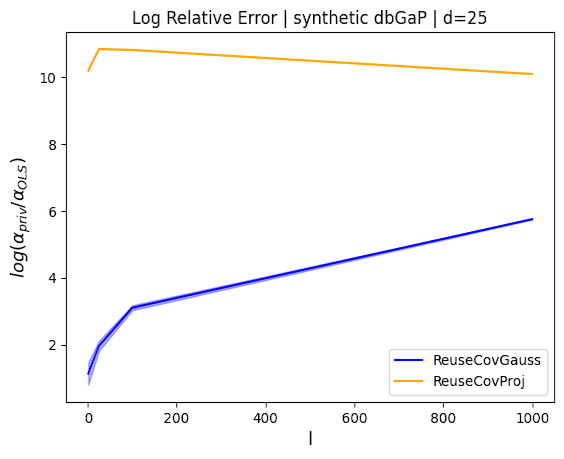}
\caption{}
\end{subfigure}
\hspace{.02\textwidth}
\begin{subfigure}[c]{.45\textwidth}
\includegraphics[width=\textwidth]{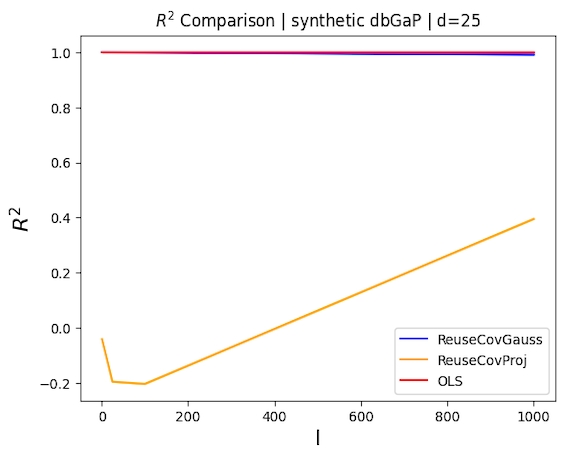}
\caption{}
\end{subfigure}
\caption{(a) shows the log of the ratio of the squared loss of the private estimator to the square loss of the OLS estimator and (b) shows the $R^2$ values. For both of these, the value was averaged over $10$ iterations. The shaded area around the lines indicates the error bars at a given value of $l$. We plot this for $l = (1,10,25,100,500,1000,5000)$ with $d=25$, $(\epsilon, \delta) = (5,\frac{1}{n^2})$ and $n = 5\cdot 1\mathrm{e}{5}$ on a synthetic dbGaP dataset.}
\label{fig:large_synth}
\end{figure}

\textbf{Additional Ablations}.
In the setting where $l < \frac{n}{\sqrt{d}}$ (Row $2$ of Table~\ref{tbl:summary}), theory suggests our algorithms should not outperform naive composition of private regressions. Since private regression methods like \algotext{DP-SGD} \cite{dp-sgd} are known to equal or outperform \algotext{SSP}, our \algotext{ReuseCov} algorithms which are variants of \algotext{SSP} might actually perform worse than composing \algotext{DP-SGD} across $l$ regressions.  We construct such a dataset and show this is in fact the case in Figure~\ref{fig:mnist} in the Appendix. 

The experiments above on genomic data generate outcomes $Y$ from a noisy linear model; in Figure~\ref{fig:mlp_vs_lin} in the Appendix we show that generating outcomes from a $2$-layer MLP with noise doesn't change the relative ordering of the algorithms by performance. 

\subsection{Limitations}
There are $3$ main criticisms that can be levied at the results in this paper. The first is that the reductions in error in the feature DP and label DP settings due to projection in Lines $3-5$ of Figure~\ref{tbl:summary} rely on sufficiently large $l: l > \frac{n}{\sqrt{d}}, l > n^{2/3}$ respectively. As a result, there are practical settings where $n$  is particularly large or $l$ is small where the theory does not suggest $\algotext{ReuseCovProj}$ will outperform $\algotext{ReuseCovGauss}$. The second critique is that this improvement requires us relaxing our notion of DP to label or feature DP. In settings where both $X$ and $Y$ are highly sensitive, if we want to guarantee full DP $\algotext{ReuseCovGauss}$ would be is our only option. Finally, all of our algorithms are variants of $\algotext{SSP}$ \cite{ssp1}, which has complexity that scales super-linearly in $d$ due to forming the covariance matrix $X^{T}X$. For high-dimensional regression settings algorithms like private SGD are preferred due to their scalability \cite{sgd}. With respect to the first critique, we note that in common settings, like in genomics, $d$ can be very large relative to $n$, and so $\frac{n}{\sqrt{d}}$ can be quite small in practice. For example, in the popular \cite{1KG} dataset used in the experiments, $n = 5008$, and $d = 78961$, and so $\frac{n}{\sqrt{d}} \approx 17$. Similiarly, if $n$ is of modest size, it is plausible that $l > n^{2/3}$. We also note that our experimental results show that the projection improves error at values of $l$ \emph{much} smaller than the theory would suggest, making this algorithm the practical choice for values of $l$ as small as $11$! Finally, the $\algotext{ReuseCovGuass}$ Algorithm is itself also a contribution of this work.  On the issue of scalability, in Subsection~\ref{subsec:complexity_sketch} we discuss how when $n > d > l$, the computational complexity of the $\algotext{ReuseCovGauss}$ algorithm is the cost of computing the covariance matrix $X^{T}X$, which is $O(nd^2)$, which we improve to $O(sd^2)$ if we use $s$ points to sketch the covariance matrix (Appendix~\ref{sec:sub_ssp}). Despite these improvements, the computational complexity of our methods is still super-linear in $d$. 

Obtaining versions of PRIMO that scale for large values of $d$, and that do not require relaxing full DP in order to improve the dependence on $l$ as we do in Section~\ref{sec:query_release}, are exciting directions for future work.


\clearpage

\bibliography{main}

\begin{thebibliography}{64}
\providecommand{\natexlab}[1]{#1}
\providecommand{\url}[1]{\texttt{#1}}
\expandafter\ifx\csname urlstyle\endcsname\relax
  \providecommand{\doi}[1]{doi: #1}\else
  \providecommand{\doi}{doi: \begingroup \urlstyle{rm}\Url}\fi

\bibitem[Agresti \& Barbara(2009)Agresti and Barbara]{social}
Alan Agresti and Finlay Barbara.
\newblock Statistical methods for the social sciences.
\newblock \emph{Pearson Prentice Hall}, 2009.

\bibitem[Ayd{\"{o}}re et~al.(2021)Ayd{\"{o}}re, Brown, Kearns, Kenthapadi, Melis, Roth, and Siva]{adaptive_proj}
Serg{\"{u}}l Ayd{\"{o}}re, William Brown, Michael Kearns, Krishnaram Kenthapadi, Luca Melis, Aaron Roth, and Amaresh~Ankit Siva.
\newblock Differentially private query release through adaptive projection.
\newblock In Marina Meila and Tong Zhang (eds.), \emph{Proceedings of the 38th International Conference on Machine Learning, {ICML} 2021, 18-24 July 2021, Virtual Event}, volume 139 of \emph{Proceedings of Machine Learning Research}, pp.\  457--467. {PMLR}, 2021.
\newblock URL \url{http://proceedings.mlr.press/v139/aydore21a.html}.

\bibitem[Bassily et~al.(2014)Bassily, Smith, and Thakurta]{priv_lb}
Raef Bassily, Adam~D. Smith, and Abhradeep Thakurta.
\newblock Private empirical risk minimization, revisited.
\newblock \emph{CoRR}, abs/1405.7085, 2014.
\newblock URL \url{http://arxiv.org/abs/1405.7085}.

\bibitem[Bassily et~al.(2019)Bassily, Feldman, Talwar, and Thakurta]{dp-sgd}
Raef Bassily, Vitaly Feldman, Kunal Talwar, and Abhradeep Thakurta.
\newblock Private stochastic convex optimization with optimal rates.
\newblock \emph{CoRR}, abs/1908.09970, 2019.
\newblock URL \url{http://arxiv.org/abs/1908.09970}.

\bibitem[Bassily et~al.(2021)Bassily, Nissim, Smith, Steinke, Stemmer, and Ullman]{ada2}
Raef Bassily, Kobbi Nissim, Adam~D. Smith, Thomas Steinke, Uri Stemmer, and Jonathan~R. Ullman.
\newblock Algorithmic stability for adaptive data analysis.
\newblock \emph{{SIAM} J. Comput.}, 50\penalty0 (3), 2021.
\newblock \doi{10.1137/16M1103646}.
\newblock URL \url{https://doi.org/10.1137/16M1103646}.

\bibitem[Beaulieu-Jones et~al.(2019)Beaulieu-Jones, Wu, Williams, Lee, Bhavnani, Byrd, and Greene]{dpgan}
Brett Beaulieu-Jones, Zhiwei Wu, Chris Williams, Ran Lee, Sanjeev Bhavnani, James Byrd, and Casey Greene.
\newblock Privacy-preserving generative deep neural networks support clinical data sharing.
\newblock \emph{Circulation: Cardiovascular Quality and Outcomes}, 12, 07 2019.
\newblock \doi{10.1161/CIRCOUTCOMES.118.005122}.

\bibitem[Blum et~al.(2011)Blum, Ligett, and Roth]{median_mech}
Avrim Blum, Katrina Ligett, and Aaron Roth.
\newblock A learning theory approach to non-interactive database privacy.
\newblock \emph{CoRR}, abs/1109.2229, 2011.
\newblock URL \url{http://arxiv.org/abs/1109.2229}.

\bibitem[Bun et~al.(2013)Bun, Ullman, and Vadhan]{buv}
Mark Bun, Jonathan~R. Ullman, and Salil~P. Vadhan.
\newblock Fingerprinting codes and the price of approximate differential privacy.
\newblock \emph{CoRR}, abs/1311.3158, 2013.
\newblock URL \url{http://arxiv.org/abs/1311.3158}.

\bibitem[Cai et~al.(2020)Cai, Wang, and Zhang]{cai2020cost}
T.~Tony Cai, Yichen Wang, and Linjun Zhang.
\newblock The cost of privacy: Optimal rates of convergence for parameter estimation with differential privacy, 2020.

\bibitem[Chaudhuri et~al.(2011)Chaudhuri, Monteleoni, and Sarwate]{sgd}
Kamalika Chaudhuri, Claire Monteleoni, and Anand~D. Sarwate.
\newblock Differentially private empirical risk minimization.
\newblock \emph{Journal of Machine Learning Research}, 12\penalty0 (29):\penalty0 1069--1109, 2011.
\newblock URL \url{http://jmlr.org/papers/v12/chaudhuri11a.html}.

\bibitem[Deng(2012)]{mnist}
Li~Deng.
\newblock The mnist database of handwritten digit images for machine learning research.
\newblock \emph{IEEE Signal Processing Magazine}, 29\penalty0 (6):\penalty0 141--142, 2012.

\bibitem[Derezinski \& Warmuth(2017)Derezinski and Warmuth]{samp_vol}
Michal Derezinski and Manfred~K. Warmuth.
\newblock Unbiased estimates for linear regression via volume sampling.
\newblock \emph{CoRR}, abs/1705.06908, 2017.
\newblock URL \url{http://arxiv.org/abs/1705.06908}.

\bibitem[Derezinski et~al.(2018)Derezinski, Warmuth, and Hsu]{samp_vol_lev}
Michal Derezinski, Manfred~K. Warmuth, and Daniel Hsu.
\newblock Tail bounds for volume sampled linear regression.
\newblock \emph{CoRR}, abs/1802.06749, 2018.
\newblock URL \url{http://arxiv.org/abs/1802.06749}.

\bibitem[Dimitrakakis et~al.(2016)Dimitrakakis, Nelson, , Zhang, Mitrokotsa, and Rubinstein]{sampling}
Christos Dimitrakakis, Blaine Nelson, , Zuhe Zhang, Aikaterini Mitrokotsa, and Benjamin Rubinstein.
\newblock Bayesian differential privacy through posterior sampling, 2016.
\newblock URL \url{https://arxiv.org/abs/1306.1066}.

\bibitem[Drineas et~al.(2006)Drineas, Mahoney, and Muthukrishnan]{samp_lev}
Petros Drineas, Michael~W. Mahoney, and S.~Muthukrishnan.
\newblock Sampling algorithms for l2 regression and applications.
\newblock In \emph{SODA '06}, 2006.

\bibitem[Dwork \& Roth(2014)Dwork and Roth]{privacybook}
Cynthia Dwork and Aaron Roth.
\newblock The algorithmic foundations of differential privacy.
\newblock \emph{Foundations and Trends in Theoretical Computer Science}, 9\penalty0 (3-4):\penalty0 211--407, 2014.
\newblock URL \url{http://dblp.uni-trier.de/db/journals/fttcs/fttcs9.html#DworkR14}.

\bibitem[Dwork et~al.(2006)Dwork, Kenthapadi, McSherry, Mironov, and Naor]{dwork2}
Cynthia Dwork, Krishnaram Kenthapadi, Frank McSherry, Ilya Mironov, and Moni Naor.
\newblock Our data, ourselves: Privacy via distributed noise generation.
\newblock In Serge Vaudenay (ed.), \emph{Advances in Cryptology - {EUROCRYPT} 2006, 25th Annual International Conference on the Theory and Applications of Cryptographic Techniques, St. Petersburg, Russia, May 28 - June 1, 2006, Proceedings}, volume 4004 of \emph{Lecture Notes in Computer Science}, pp.\  486--503. Springer, 2006.
\newblock \doi{10.1007/11761679\_29}.
\newblock URL \url{https://doi.org/10.1007/11761679\_29}.

\bibitem[Dwork et~al.(2015{\natexlab{a}})Dwork, Feldman, Hardt, Pitassi, Reingold, and Roth]{ada}
Cynthia Dwork, Vitaly Feldman, Moritz Hardt, Toniann Pitassi, Omer Reingold, and Aaron Roth.
\newblock Generalization in adaptive data analysis and holdout reuse.
\newblock \emph{CoRR}, abs/1506.02629, 2015{\natexlab{a}}.
\newblock URL \url{http://arxiv.org/abs/1506.02629}.

\bibitem[Dwork et~al.(2015{\natexlab{b}})Dwork, Nikolov, and Talwar]{2way_convex}
Cynthia Dwork, Aleksandar Nikolov, and Kunal Talwar.
\newblock Efficient algorithms for privately releasing marginals via convex relaxations.
\newblock \emph{Discret. Comput. Geom.}, 53\penalty0 (3):\penalty0 650--673, 2015{\natexlab{b}}.
\newblock \doi{10.1007/s00454-015-9678-x}.
\newblock URL \url{https://doi.org/10.1007/s00454-015-9678-x}.

\bibitem[Edmonds et~al.(2019)Edmonds, Nikolov, and Ullman]{fac_mech}
Alexander Edmonds, Aleksandar Nikolov, and Jonathan~R. Ullman.
\newblock The power of factorization mechanisms in local and central differential privacy.
\newblock \emph{CoRR}, abs/1911.08339, 2019.
\newblock URL \url{http://arxiv.org/abs/1911.08339}.

\bibitem[Esfandiari et~al.(2021)Esfandiari, Mirrokni, Syed, and Vassilvitskii]{lab2}
Hossein Esfandiari, Vahab~S. Mirrokni, Umar Syed, and Sergei Vassilvitskii.
\newblock Label differential privacy via clustering.
\newblock \emph{CoRR}, abs/2110.02159, 2021.
\newblock URL \url{https://arxiv.org/abs/2110.02159}.

\bibitem[Fairley et~al.(2019)Fairley, Lowy-Gallego, Perry, and Flicek]{1KG}
Susan Fairley, Ernesto Lowy-Gallego, Emily Perry, and Paul Flicek.
\newblock {The International Genome Sample Resource (IGSR) collection of open human genomic variation resources}.
\newblock \emph{Nucleic Acids Research}, 48\penalty0 (D1):\penalty0 D941--D947, 10 2019.
\newblock ISSN 0305-1048.
\newblock \doi{10.1093/nar/gkz836}.
\newblock URL \url{https://doi.org/10.1093/nar/gkz836}.

\bibitem[Foulds et~al.(2016)Foulds, Geumlek, Welling, and Chaudhuri]{ssp2}
James~R. Foulds, Joseph Geumlek, Max Welling, and Kamalika Chaudhuri.
\newblock On the theory and practice of privacy-preserving bayesian data analysis.
\newblock \emph{CoRR}, abs/1603.07294, 2016.
\newblock URL \url{http://arxiv.org/abs/1603.07294}.

\bibitem[Gall(2012)]{fast_mat}
Fran{\c{c}}ois~Le Gall.
\newblock Faster algorithms for rectangular matrix multiplication.
\newblock \emph{CoRR}, abs/1204.1111, 2012.
\newblock URL \url{http://arxiv.org/abs/1204.1111}.

\bibitem[Gelman \& Loken(2014)Gelman and Loken]{gelman}
Andrew Gelman and Eric~R. Loken.
\newblock The statistical crisis in science.
\newblock \emph{American Scientist}, 102:\penalty0 460, 2014.
\newblock URL \url{https://api.semanticscholar.org/CorpusID:146855235}.

\bibitem[Ghazi et~al.(2021)Ghazi, Golowich, Kumar, Manurangsi, and Zhang]{lab1}
Badih Ghazi, Noah Golowich, Ravi Kumar, Pasin Manurangsi, and Chiyuan Zhang.
\newblock On deep learning with label differential privacy.
\newblock \emph{CoRR}, abs/2102.06062, 2021.
\newblock URL \url{https://arxiv.org/abs/2102.06062}.

\bibitem[Golub \& Van~Loan(1996)Golub and Van~Loan]{GoluVanl96}
Gene~H. Golub and Charles~F. Van~Loan.
\newblock \emph{Matrix Computations}.
\newblock The Johns Hopkins University Press, third edition, 1996.

\bibitem[Gupta et~al.(2011)Gupta, Roth, and Ullman]{ullman2}
Anupam Gupta, Aaron Roth, and Jonathan Ullman.
\newblock Iterative constructions and private data release, 2011.

\bibitem[Hager(2001)]{hager}
William~W. Hager.
\newblock Minimizing a quadratic over a sphere.
\newblock \emph{SIAM J. Optim.}, 12:\penalty0 188--208, 2001.

\bibitem[Hardt \& Rothblum(2010)Hardt and Rothblum]{pmw}
Moritz Hardt and Guy~N. Rothblum.
\newblock A multiplicative weights mechanism for privacy-preserving data analysis.
\newblock In \emph{2010 IEEE 51st Annual Symposium on Foundations of Computer Science}, pp.\  61--70, 2010.
\newblock \doi{10.1109/FOCS.2010.85}.

\bibitem[Homer et~al.(2008)Homer, Szelinger, Redman, Duggan, Tembe, Muehling, Pearson, Stephan, Nelson, and Craig]{homer}
Nils Homer, Szabolcs Szelinger, Margot Redman, David Duggan, Waibhav Tembe, Jill Muehling, John~V. Pearson, Dietrich~A. Stephan, Stanley~F. Nelson, and David~W. Craig.
\newblock Resolving individuals contributing trace amounts of dna to highly complex mixtures using high-density snp genotyping microarrays.
\newblock \emph{PLOS Genetics}, 4\penalty0 (8):\penalty0 1--9, 08 2008.
\newblock \doi{10.1371/journal.pgen.1000167}.
\newblock URL \url{https://doi.org/10.1371/journal.pgen.1000167}.

\bibitem[Hu et~al.(2021)Hu, Salcic, Sun, Dobbie, Yu, and Zhang]{mi_survey}
Hongsheng Hu, Zoran Salcic, Lichao Sun, Gillian Dobbie, Philip~S Yu, and Xuyun Zhang.
\newblock Membership inference attacks on machine learning: A survey.
\newblock \emph{ACM Computing Surveys (CSUR)}, 2021.

\bibitem[Kasiviswanathan et~al.(2010)Kasiviswanathan, Rudelson, Smith, and Ullman]{lower_bnd_2way}
Shiva~Prasad Kasiviswanathan, Mark Rudelson, Adam~D. Smith, and Jonathan~R. Ullman.
\newblock The price of privately releasing contingency tables and the spectra of random matrices with correlated rows.
\newblock In Leonard~J. Schulman (ed.), \emph{Proceedings of the 42nd {ACM} Symposium on Theory of Computing, {STOC} 2010, Cambridge, Massachusetts, USA, 5-8 June 2010}, pp.\  775--784. {ACM}, 2010.
\newblock \doi{10.1145/1806689.1806795}.
\newblock URL \url{https://doi.org/10.1145/1806689.1806795}.

\bibitem[Kearns(1993)]{kearns}
Michael~J. Kearns.
\newblock Efficient noise-tolerant learning from statistical queries.
\newblock In S.~Rao Kosaraju, David~S. Johnson, and Alok Aggarwal (eds.), \emph{Proceedings of the Twenty-Fifth Annual {ACM} Symposium on Theory of Computing, May 16-18, 1993, San Diego, CA, {USA}}, pp.\  392--401. {ACM}, 1993.
\newblock \doi{10.1145/167088.167200}.
\newblock URL \url{https://doi.org/10.1145/167088.167200}.

\bibitem[Kifer et~al.(2012)Kifer, Smith, and Thakurta]{obj}
Daniel Kifer, Adam Smith, and Abhradeep Thakurta.
\newblock Private convex empirical risk minimization and high-dimensional regression.
\newblock In Shie Mannor, Nathan Srebro, and Robert~C. Williamson (eds.), \emph{Proceedings of the 25th Annual Conference on Learning Theory}, volume~23 of \emph{Proceedings of Machine Learning Research}, pp.\  25.1--25.40, Edinburgh, Scotland, 25--27 Jun 2012. PMLR.
\newblock URL \url{https://proceedings.mlr.press/v23/kifer12.html}.

\bibitem[Korthauer et~al.(2019)Korthauer, Kimes, Duvallet, Reyes, Subramanian, Teng, Shukla, Alm, and Hicks]{mult}
Keegan Korthauer, Patrick~K. Kimes, Claire Duvallet, Alejandro Reyes, Ayshwarya Subramanian, Mingxiang Teng, Chinmay Shukla, Eric~J. Alm, and Stephanie~C. Hicks.
\newblock A practical guide to methods controlling false discoveries in computational biology.
\newblock \emph{Genome Biology}, 20\penalty0 (1):\penalty0 118, 2019.
\newblock \doi{10.1186/s13059-019-1716-1}.
\newblock URL \url{https://doi.org/10.1186/s13059-019-1716-1}.

\bibitem[Krapohl et~al.(2018)Krapohl, Patel, Newhouse, Curtis, von Stumm, Dale, Zabaneh, Breen, O'Reilly, and Plomin]{prs_linear}
E.~Krapohl, H.~Patel, S.~Newhouse, C~J Curtis, S.~von Stumm, P~S Dale, D.~Zabaneh, G.~Breen, P~F O'Reilly, and R.~Plomin.
\newblock Multi-polygenic score approach to trait prediction.
\newblock \emph{Molecular Psychiatry}, 23\penalty0 (5):\penalty0 1368--1374, 2018.
\newblock \doi{10.1038/mp.2017.163}.
\newblock URL \url{https://doi.org/10.1038/mp.2017.163}.

\bibitem[Lee \& Sun(2015)Lee and Sun]{samp_spectral}
Yin~Tat Lee and He~Sun.
\newblock Constructing linear-sized spectral sparsification in almost-linear time.
\newblock \emph{CoRR}, abs/1508.03261, 2015.
\newblock URL \url{http://arxiv.org/abs/1508.03261}.

\bibitem[Mahoney(2011)]{mahoney}
Michael~W. Mahoney.
\newblock Randomized algorithms for matrices and data.
\newblock \emph{CoRR}, abs/1104.5557, 2011.
\newblock URL \url{http://arxiv.org/abs/1104.5557}.

\bibitem[Mailman et~al.(2007)Mailman, Feolo, Jin, Kimura, Tryka, Bagoutdinov, Hao, Kiang, Paschall, Phan, Popova, Pretel, Ziyabari, Lee, Shao, Wang, Sirotkin, Ward, Kholodov, Zbicz, Beck, Kimelman, Shevelev, Preuss, Yaschenko, Graeff, Ostell, and Sherry]{dbgap}
Matthew~D Mailman, Michael Feolo, Yan Jin, Michi Kimura, Kimberly Tryka, Rinat Bagoutdinov, Li~Hao, Alex Kiang, Justin Paschall, Lon Phan, Natalia Popova, Sherri Pretel, Luda Ziyabari, Ming Lee, Yongzhao Shao, Zhining~Y Wang, Karl Sirotkin, Michael Ward, Maxim Kholodov, Karen Zbicz, Judith Beck, Michael Kimelman, Sergey Shevelev, Don Preuss, Elena Yaschenko, Alexander Graeff, James Ostell, and Stephen~T Sherry.
\newblock The ncbi dbgap database of genotypes and phenotypes.
\newblock \emph{Nature genetics}, 39\penalty0 (10):\penalty0 1181--1186, 2007.

\bibitem[McKenna et~al.(2018)McKenna, Miklau, Hay, and Machanavajjhala]{hdmm}
Ryan McKenna, Gerome Miklau, Michael Hay, and Ashwin Machanavajjhala.
\newblock Optimizing error of high-dimensional statistical queries under differential privacy.
\newblock \emph{CoRR}, abs/1808.03537, 2018.
\newblock URL \url{http://arxiv.org/abs/1808.03537}.

\bibitem[McKenna et~al.(2019)McKenna, Sheldon, and Miklau]{hdmm2}
Ryan McKenna, Daniel Sheldon, and Gerome Miklau.
\newblock Graphical-model based estimation and inference for differential privacy.
\newblock \emph{CoRR}, abs/1901.09136, 2019.
\newblock URL \url{http://arxiv.org/abs/1901.09136}.

\bibitem[Milionis et~al.(2022)Milionis, Kalavasis, Fotakis, and Ioannidis]{rand_design}
Jason Milionis, Alkis Kalavasis, Dimitris Fotakis, and Stratis Ioannidis.
\newblock Differentially private regression with unbounded covariates.
\newblock In Gustau Camps-Valls, Francisco J.~R. Ruiz, and Isabel Valera (eds.), \emph{Proceedings of The 25th International Conference on Artificial Intelligence and Statistics}, volume 151 of \emph{Proceedings of Machine Learning Research}, pp.\  3242--3273. PMLR, 28--30 Mar 2022.
\newblock URL \url{https://proceedings.mlr.press/v151/milionis22a.html}.

\bibitem[Mironov(2017)]{renyi}
Ilya Mironov.
\newblock Renyi differential privacy.
\newblock \emph{CoRR}, abs/1702.07476, 2017.
\newblock URL \url{http://arxiv.org/abs/1702.07476}.

\bibitem[Muthukrishnan \& Nikolov(2012)Muthukrishnan and Nikolov]{her_lb}
S.~Muthukrishnan and Aleksandar Nikolov.
\newblock Optimal private halfspace counting via discrepancy.
\newblock In Howard~J. Karloff and Toniann Pitassi (eds.), \emph{Proceedings of the 44th Symposium on Theory of Computing Conference, {STOC} 2012, New York, NY, USA, May 19 - 22, 2012}, pp.\  1285--1292. {ACM}, 2012.
\newblock \doi{10.1145/2213977.2214090}.
\newblock URL \url{https://doi.org/10.1145/2213977.2214090}.

\bibitem[Neel et~al.(2019)Neel, Roth, and Wu]{heuristics}
Seth~V. Neel, Aaron~L. Roth, and Zhiwei~Steven Wu.
\newblock How to use heuristics for differential privacy.
\newblock In \emph{2019 IEEE 60th Annual Symposium on Foundations of Computer Science (FOCS)}, pp.\  72--93, 2019.
\newblock \doi{10.1109/FOCS.2019.00014}.

\bibitem[Nikolov et~al.(2013)Nikolov, Talwar, and Zhang]{proj}
Aleksandar Nikolov, Kunal Talwar, and Li~Zhang.
\newblock The geometry of differential privacy: the sparse and approximate cases.
\newblock In \emph{STOC '13}, 2013.

\bibitem[Olagoke et~al.(2023)Olagoke, Vadhan, and Neel]{lukman}
Lukman Olagoke, Salil Vadhan, and Seth Neel.
\newblock Black-box training data identification in gans via detector networks, 2023.

\bibitem[Papernot et~al.(2016)Papernot, Abadi, Erlingsson, Goodfellow, and Talwar]{lab3}
Nicolas Papernot, Martín Abadi, Úlfar Erlingsson, Ian Goodfellow, and Kunal Talwar.
\newblock Semi-supervised knowledge transfer for deep learning from private training data, 2016.
\newblock URL \url{https://arxiv.org/abs/1610.05755}.

\bibitem[Pattee \& Pan(2020)Pattee and Pan]{prs_2}
Jack Pattee and Wei Pan.
\newblock Penalized regression and model selection methods for polygenic scores on summary statistics.
\newblock \emph{PLOS Computational Biology}, 16\penalty0 (10):\penalty0 1--27, 10 2020.
\newblock \doi{10.1371/journal.pcbi.1008271}.
\newblock URL \url{https://doi.org/10.1371/journal.pcbi.1008271}.

\bibitem[Priv{\'{e}}(2022)]{biobank}
Florian Priv{\'{e}}.
\newblock Using the {UK} biobank as a global reference of worldwide populations: application to measuring ancestry diversity from {GWAS} summary statistics.
\newblock \emph{Bioinform.}, 38\penalty0 (13):\penalty0 3477--3480, 2022.
\newblock \doi{10.1093/bioinformatics/btac348}.
\newblock URL \url{https://doi.org/10.1093/bioinformatics/btac348}.

\bibitem[Raskutti et~al.(2009)Raskutti, Wainwright, and Yu]{lem_proj}
Garvesh Raskutti, Martin~J. Wainwright, and Bin Yu.
\newblock Minimax rates of estimation for high-dimensional linear regression over $\ell_q$-balls, 2009.
\newblock URL \url{https://arxiv.org/abs/0910.2042}.

\bibitem[Shokri et~al.(2016)Shokri, Stronati, and Shmatikov]{mi}
Reza Shokri, Marco Stronati, and Vitaly Shmatikov.
\newblock Membership inference attacks against machine learning models.
\newblock \emph{CoRR}, abs/1610.05820, 2016.
\newblock URL \url{http://arxiv.org/abs/1610.05820}.

\bibitem[Thaler et~al.(2012)Thaler, Ullman, and Vadhan]{2way_fast}
Justin Thaler, Jonathan~R. Ullman, and Salil~P. Vadhan.
\newblock Faster algorithms for privately releasing marginals.
\newblock In Artur Czumaj, Kurt Mehlhorn, Andrew~M. Pitts, and Roger Wattenhofer (eds.), \emph{Automata, Languages, and Programming - 39th International Colloquium, {ICALP} 2012, Warwick, UK, July 9-13, 2012, Proceedings, Part {I}}, volume 7391 of \emph{Lecture Notes in Computer Science}, pp.\  810--821. Springer, 2012.
\newblock \doi{10.1007/978-3-642-31594-7\_68}.
\newblock URL \url{https://doi.org/10.1007/978-3-642-31594-7\_68}.

\bibitem[Torkzadehmahani et~al.(2020)Torkzadehmahani, Kairouz, and Paten]{dc_gan}
Reihaneh Torkzadehmahani, Peter Kairouz, and Benedict Paten.
\newblock {DP-CGAN:} differentially private synthetic data and label generation.
\newblock \emph{CoRR}, abs/2001.09700, 2020.
\newblock URL \url{https://arxiv.org/abs/2001.09700}.

\bibitem[Tropp(2010)]{tropp}
Joel~A. Tropp.
\newblock Improved analysis of the subsampled randomized hadamard transform.
\newblock \emph{CoRR}, abs/1011.1595, 2010.
\newblock URL \url{http://arxiv.org/abs/1011.1595}.

\bibitem[Ullman \& Vadhan(2011)Ullman and Vadhan]{2way_hard}
Jonathan~R. Ullman and Salil~P. Vadhan.
\newblock Pcps and the hardness of generating private synthetic data.
\newblock In Yuval Ishai (ed.), \emph{Theory of Cryptography - 8th Theory of Cryptography Conference, {TCC} 2011, Providence, RI, USA, March 28-30, 2011. Proceedings}, volume 6597 of \emph{Lecture Notes in Computer Science}, pp.\  400--416. Springer, 2011.
\newblock \doi{10.1007/978-3-642-19571-6\_24}.
\newblock URL \url{https://doi.org/10.1007/978-3-642-19571-6\_24}.

\bibitem[Vadhan(2017)]{complexity}
Salil~P. Vadhan.
\newblock The complexity of differential privacy.
\newblock In \emph{Tutorials on the Foundations of Cryptography}, 2017.

\bibitem[Vershynin(2019)]{vershynin}
Roman Vershynin.
\newblock \emph{High-Dimensional Probability}.
\newblock Cambridge University Press, 2019.
\newblock URL \url{https://www.math.uci.edu/~rvershyn/papers/HDP-book/HDP-book.pdf}.

\bibitem[Vu \& Slavkovic(2009)Vu and Slavkovic]{ssp1}
Duy Vu and Aleksandra Slavkovic.
\newblock Differential privacy for clinical trial data: Preliminary evaluations.
\newblock In \emph{2009 IEEE International Conference on Data Mining Workshops}, pp.\  138--143, 2009.
\newblock \doi{10.1109/ICDMW.2009.52}.

\bibitem[Wang(2018)]{revisit}
Yu-Xiang Wang.
\newblock Revisiting differentially private linear regression: optimal and adaptive prediction and estimation in unbounded domain, 2018.
\newblock URL \url{https://arxiv.org/abs/1803.02596}.

\bibitem[Wang et~al.(2018)Wang, Balle, and Kasiviswanathan]{subsamp}
Yu{-}Xiang Wang, Borja Balle, and Shiva~Prasad Kasiviswanathan.
\newblock Subsampled r{\'{e}}nyi differential privacy and analytical moments accountant.
\newblock \emph{CoRR}, abs/1808.00087, 2018.
\newblock URL \url{http://arxiv.org/abs/1808.00087}.

\bibitem[Yoon et~al.(2019)Yoon, Jordon, and van~der Schaar]{pate_gan}
Jinsung Yoon, James Jordon, and Mihaela van~der Schaar.
\newblock {PATE}-{GAN}: Generating synthetic data with differential privacy guarantees.
\newblock In \emph{International Conference on Learning Representations}, 2019.
\newblock URL \url{https://openreview.net/forum?id=S1zk9iRqF7}.

\bibitem[Yousefpour et~al.(2021)Yousefpour, Shilov, Sablayrolles, Testuggine, Prasad, Malek, Nguyen, Ghosh, Bharadwaj, Zhao, Cormode, and Mironov]{opacus}
Ashkan Yousefpour, Igor Shilov, Alexandre Sablayrolles, Davide Testuggine, Karthik Prasad, Mani Malek, John Nguyen, Sayan Ghosh, Akash Bharadwaj, Jessica Zhao, Graham Cormode, and Ilya Mironov.
\newblock Opacus: {U}ser-friendly differential privacy library in {PyTorch}.
\newblock \emph{arXiv preprint arXiv:2109.12298}, 2021.

\end{thebibliography}
\bibliographystyle{tmlr}
\clearpage

\section{Appendix} 
\subsection{Additional Related Work}
\paragraph{Query Release.}
In Subsection~\ref{sec:query_release} we show that privately computing the association term $\frac{1}{n}X^{T}Y$ is equivalent to the problem of differentially privately releasing a set of $l \cdot d$ low-sensitivity queries \cite{ada2}. While less is known about the optimal $l_2$ error for arbitrary low-sensitivity queries, it is clear that geometric techniques based on factorization and projection do not (at least obviously) apply, since there is no corresponding notion of the query matrix $A_\mathcal{Q}$ \cite{fac_mech}.  In the case where $X, Y$ are both private, this corresponds to releasing a subset of $l \cdot d$ $2$-way marginal queries over $d + l$ dimensions, which is well-studied \cite{2way_convex, 2way_fast, 2way_hard}. In the related work section of the Appendix we discuss why applying existing algorithms for private release of marginals seems unlikely be able to both improve over the Gaussian Mechanism and achieve non-trivial error. However, in the less restrictive but still practically relevant setting where the labels $Y$ are public, computing the association term $\frac{1}{n}X^{T}Y$ is equivalent to the problem of differentially privately releasing a special class of  $dl$ ``low-sensitivity" queries we term ``inner product queries" (Definition~\ref{def:inner_prod}).  For this special class of queries we can adapt the projected Gaussian Mechanism of \cite{adaptive_proj} that is optimal for linear queries under $l_2$ error in the sparse regime.  We are able to obtain greatly improved results over the naive Gaussian Mechanism in the regime where $l  > \frac{n}{\sqrt{d}}$, as summarised in Figure~\ref{tbl:summary} and presented in Subsection~\ref{sec:query_release}. This is a key regime of interest, particularly for genomic data, where $d$ the number of SNPs could be in the hundreds of millions, $n$ the population size in the hundreds of thousands ($\sim 400K$ for UK Biobank \cite{biobank} one of the most popular databases), and $l$ the number of recorded phenotypes could be in the thousands.

\paragraph{Projection mechanisms and Algorithm~\ref{alg:proj}.}Since we make heavy use of projection mechanisms in Subsection~\ref{sec:query_release} in the \emph{public label} setting, we elaborate on the difference between the existing methods and our Algorithm~\ref{alg:proj}. The most technically related work to our Algorithm~\ref{alg:proj} is \cite{adaptive_proj}, which is itself a variant of the projection mechanism of \cite{proj}. There are several key differences in our analysis and application.The projection in \cite{proj} is designed for linear queries over a discrete domain, and runs in time polynomial in the domain size. Algorithm~\ref{alg:proj} allows continuous $\mathcal{X}, \mathcal{Y}$ and runs in time polynomial in $n, d, l$. Like the relaxed projection mechanism \cite{adaptive_proj}, when our data is discrete we relax our data domain to be continuous in order to compute the projection more efficiently. Unlike in their setting which attempts to handle general linear queries, due to the linear structure of inner product queries our projection can be computed in polynomial time via linear regression over the $l_2$ ball, as opposed to solving a possibly non-convex optimization problem. Moreover, due to this special structure we can use the same geometric techniques as in \cite{proj} to obtain theoretical accuracy guarantees (Theorem~\ref{thm:proj-acc}). 
\paragraph{Private release of $2$-way marginals}
Consider the problem of privately releasing all $2$-way marginals over $n$ points in $\{0,1\}^{d + l}$. Theorem 5.7 in \cite{2way_convex} gives a polynomial time algorithm based on relaxed projections that achieves mean squared error $\tilde{O}(n\sqrt{d + l})$, which matches the best known information theoretic upper bound \cite{2way_convex}, although there is a small gap to the existing lower bound $\min(n, (d+l)^{2})$. This relaxed projection algorithm outperforms the Gaussian Mechanism when $n \sqrt{d+l} < ld  \implies n < \frac{dl}{\sqrt{d+l}} \implies d \sqrt{d+l} > n$. Since the mean squared error of Algorithm~\ref{alg:reusecov} that used this projection as a subroutine is at least $\frac{d\sqrt{d+l}}{n}$, this means that in the regime where the projection outperforms the Gaussian Mechanism, we do not achieve mean squared error $< 1$ in our regression. 

This suggests that, at least using the existing error analysis of $\algotext{SSP}$ from \cite{revisit}, it seems unlikely that by applying specialized algorithms for private release of $2$-way marginals to compute the $\frac{1}{n}X^{T}Y$ term subject to differential privacy in $(X, Y)$, e.g. the label private setting, we can both improve over the Gaussian Mechanism and achieve non-trivial error. 

\paragraph{Linear Queries under $l_\infty$-loss.}Beyond $2$-way marginals, the problem of privately releasing large numbers of linear queries (Definition~\ref{def:linear}) has been studied extensively. It is known that the worst case error is bounded by $\min(\frac{\sqrt{\log(|\mathcal{Q}|)}(\log|\mathcal{X}|\log(1/\delta))^{1/4}}{\sqrt{n\epsilon}}, \frac{\sqrt{|\mathcal{Q}|\log(1/\delta)}}{\epsilon n})$. The first term, which dominates in the so-called low-accuracy or ``sparse'' (\cite{proj}) regime, is achieved by the \emph{PrivateMultiplicativeWeights}  algorithm of \cite{pmw}, which is optimal over worst case workloads \cite{buv}. However, this algorithm has running time exponential in the data dimension, which is unavoidable \cite{complexity} over worst case $\mathcal{Q}$. The second term, which dominates for $n \gg \frac{|\mathcal{Q}|}{\log|\mathcal{Q}|\log|\mathcal{X}|^{1/4}}$, the ``high accuracy'' regime, is achieved by the simple and efficient Gaussian Mechanism \cite{privacybook}, which is also optimal over worst-case sets of queries $\mathcal{Q}$ \cite{buv}.

\paragraph{Linear Queries under $l_2$-loss.}For the $l_2$ error, in the high accuracy $n \gg |\mathcal{Q}|$ regime the factorization mechanism achieves error that is exactly tight for any workload of linear queries $\mathcal{Q}$ up to a factor of $\log(1/\delta)$, although it is not efficient (Theorem~ \cite{fac_mech}). In the low-accuracy regime, the algorithm of \cite{proj} that couples careful addition of correlated Gaussian noise (akin to the factorization mechanism) with an $l_1$-ball projection step achieves error within log factors of what is (a slight variant) of a quantity known as the \emph{hereditary discrepancy} $\text{opt}_{\epsilon, \delta}(A, n)$ (Theorem  \cite{proj}). This quantity is a known lower bound on the error of any $(\epsilon, \delta)$ mechanism for answering linear queries \cite{her_lb}, and so the upper bound is tight up to log factors in $|\mathcal{Q}|, |\mathcal{X}|$. Theorem $21$ in \cite{proj} analyzes the simple projection mechanism that adds independent Gaussian noise and projects rather than first performing the decomposition step that utilizes correlated Gaussian noise, achieving error $O(nd\log(1/\delta)\sqrt{\log|X|}/\epsilon)$, which matches the best known (worst case over $\mathcal{Q}$) upper bound for the sparse $n < d$ case \cite{ullman2}. In our Theorem~\ref{thm:proj-acc} we give such a universal upper bound, rather than one that depends on the hereditary discrepancy of the matrix $Y$. While the bound can of course be improved for a specific set of outcomes $Y$ by the addition of the decomposition step to the projection algorithm, we omit this step in favor of a simpler algorithm with more directly comparable bounds to existing private regression algorithms. 

While the information-theoretic $l_\infty$ or $l_2$ error achievable for linear queries is well-understood \cite{fac_mech, proj, buv}, as synthetic data algorithms like \emph{PrivateMultiplicativeWeights} and \emph{MedianMechanism}, or the factorization or projection mechanisms are in general inefficient, there are many open problems pertaining to developing efficient algorithms for specific query classes, or heuristic approaches that work better in practice. Examples of these approaches along the lines of the factorization mechanism \cite{hdmm, hdmm2},  efficient approximations of the projection mechanism \cite{adaptive_proj, 2way_convex}, and using heuristic techniques from distribution learning in the framework of iterative synthetic data algorithms \cite{pate_gan, dc_gan, dpgan, heuristics}. 

\paragraph{Sub-sampled Linear Regression.} 
In Subsection~\ref{sec:sub_ssp} we analyze $\algotext{SSP}$ where we first sub-sample a random set of $s$ points without replacement, and use this sub-sample to compute the noisy covariance matrix. Sub-sampled linear regression has been studied extensively absent privacy, where it is known that uniform sub-sampling is sub-optimal in that it produces biased estimates of the OLS estimator, and performs poorly in the presence of high-leverage points \cite{samp_vol_lev}. To address these shortcomings, techniques based on leverage score sampling \cite{samp_lev}, volume-based sampling \cite{samp_vol}\cite{samp_vol_lev}, and spectral sparsification \cite{samp_spectral} have been developed. Crucially, in these methods the probability of a point being sub-sampled is data-dependent, and so they are (not obviously) compatible with differential privacy. 
\subsection{Definitions}
Throughout the paper we make heavy use of common matrix and vector norms. For a vector $v \in \mathbb{R}^d$ and matrix $A \in \mathbb{R}^{d\times d}, ||v||_{A}^2 \defeq v^{t}Av, \normreg{A}_{2}^2 = \lambda_{\max}(A^{T}A), \norm{v}^2 = \sum_{k = 1}^{d}v_k^2, \normreg{A}_F^2 = \sum_{i,j \in [d]}a_{ij}^2$.

\begin{definition}{Statistical Query \cite{kearns}}
\label{def:linear}
Let $D \in \mathcal{X}^n$ a dataset. A linear query is a function $q: \mathcal{X} \to [0, 1]$, where 
$q(D) \defeq \frac{1}{n}\sum_{x_i \in D}^{n}q(x_i)$. 
\end{definition}

\begin{definition}{\cite{lower_bnd_2way}}
Let $\mathcal{X} = \{0,1\}^m$, and $\mathcal{Q}_k = \{q_{ij}\}_{1 \leq i_1 < i_2 < i_k \leq m},$ where $q_{ij}(x) \defeq \prod_{u = 1}^{k}x_{i_u}$. Then the class $\mathcal{Q}_k$ are called $k-$way marginals. 
\end{definition}

Lastly, we define a new class of ``low-sensitivity'' queries \cite{ada2} we call ``inner product'' queries that we will make heavy use of in Subsection~\ref{sec:query_release}.
\begin{definition}{Inner Product Query.}
\label{def:inner_prod}
Let $\mathcal{X} \subset \mathbb{R}^{d}$ and $X \in \mathcal{X}^n \subset \mathbb{R}^{d \times n}$ a dataset with column $j$ corresponding to $x_j \in \mathcal{X}$. Let $y \in \mathbb{R}^{n}$, and $i \in [d]$. Then the pair $(i, y)$ defines an inner product query $q_{(i, y)}: \mathcal{X}^n \to \mathbb{R}$, 
$$
q_{(i, y)}(X) \defeq \frac{1}{n}\sum_{j=1}^{n}X_{ij}y_{j} = \frac{1}{n}e_{i}^{T} X y,
$$
where $e_i$ is the $i^{th}$ basis vector in $\mathbb{R}^d$. 
\end{definition}
Changing one row $j$ of $X$, changes at most one summand $\frac{1}{n}X_{ij}y_{j}$ of $q_{(i, y)}(X)$, and hence the $l_2$ sensitivity 
$\Delta_2(q_{(i,y)}) \leq \frac{1}{n}|X_{ij}||y_j| \leq \frac{\infnormX \normY}{n}$, a fact we will use throughout.

\subsection{Lemmas}

\begin{lemma}[\cite{lem_proj}]\label{lem:proj_lemma}
Let $K \subset \mathbb{R}^d$ be a symmetric convex body, let $g \in K$, and $\tilde{g} = g + w$ for some $w \in \mathbb{R}^d$. Then if $\hat{g} = \text{argmin}_{g' \in K}\norm{\tilde{g}-g'}^2$, then
$$
\norm{\hat{g}-g}^2 \leq \min \{4\norm{w}^2, 4\normreg{w}_{K^{\circ}}\}
$$
\end{lemma}

\begin{lemma}{\cite{tropp}}
\label{lem:tropp1}
Let $Z_i \in \mathbb{S}_{d}^{+}$, and sample $Z_1, \ldots Z_s$ without replacement from $\{Z_1, \ldots Z_n\}$. Suppose $\mathbb{E}[Z_i] = I_{d}$, and $\max_{i \in [n]}\lambda_{\max}(Z_i) \leq B$. Then for $\delta = \sqrt{\frac{2B\log(2d/\rho)}{s}}$, with probability $1-\rho$:
\begin{align*}
\lambda_{\max}(\frac{1}{s}\sum_{i = 1}^{s}Z_i) < 1 + \delta, \; \; \lambda_{\min}(\frac{1}{s}\sum_{i = 1}^{s}Z_i)   > 1 - \delta 
\end{align*}
\end{lemma}

\begin{lemma}[folklore e.g. \cite{subsamp}]\label{lem:secret}
Given a dataset $\mathcal{X}^n$ of $n$ points and an $(\epsilon, \delta)-$DP mechanism $M$. Let the procedure \textbf{subsample} take a random subset of $s$ points from $\mathcal{X}^n$ without replacement. Then if $\gamma = s/n$, the procedure $M \circ \textbf{subsample}$ is $(O(\gamma \epsilon), \gamma \delta)-$DP for sufficiently small $\epsilon$.
\end{lemma}

The following lemma is used repeatedly in analyzing the accuracy of all $\algotext{SSP}$ variants. 
\begin{lemma}
\label{lem:sherman}
Let $A, B$ invertible matrices in $\mathbb{R}^{n\times n}$, and $v, c$  vectors $\in \mathbb{R}^{n}$. 

Then 
$$ A^{-1}v - (A+B)^{-1}(v + c) = (A+B)^{-1}BA^{-1}v  - (A+B)^{-1}c $$

\end{lemma}
\begin{proof}
Expanding we get that: 
$$ A^{-1}v - (A+B)^{-1}(v + c) = (A^{-1} - (A+B)^{-1})v - (A+B)^{-1}c, $$
so it suffices to show that: 

$$(A^{-1} - (A+B)^{-1})v = (A+B)^{-1}BA^{-1}v$$
Now the Woodbury formula tells us that $(A+B)^{-1} = A^{-1} - (A + AB^{-1}A)^{-1},$ hence 
$$(A^{-1} - (A+B)^{-1})v = (A + AB^{-1}A)^{-1}v = (A(I + B^{-1}A))^{-1}v$$
Then since: 
\begin{equation*}
    (A(I + B^{-1}A))^{-1} = (I + B^{-1}A)^{-1}A^{-1} = (B^{-1}(B + A))^{-1A^{-1}} = (B+A)^{-1}BA^{-1},
\end{equation*} 
we are done. 
\end{proof}

\subsection{Proofs from Subsection~\ref{sec:query_release}}
\label{subsec:app_query}
\paragraph{\textbf{Computing the association term.}}
Consider entry $(k, j)$ of $\frac{1}{n}X^{T}Y$ for $k \in [d], j \in [l]$, $a_{kj} = (\frac{1}{n}X^{T}Y)_{kj} = e_k\frac{1}{n}X^{T}y_j = q_{(k, y_j)}(X)$, $q_{(k, y_j)}$ is the inner product query of Definition~\ref{def:inner_prod}. Thus privately computing $\frac{1}{n}X^{T}Y$, is equivalent to privately releasing answers to $dl$ inner product queries $\{q_{(k, y_j)}(X)\}_{k \in [d], j \in [l]}$.

It will be convenient for us to write $q_{(k, y_j)}(X)$ as a single inner product. Let $\vecx = (x_{11}, \ldots, x_{1n}, x_{21}, \ldots x_{d1}, \ldots x_{dn}) \in \mathbb{R}^{nd}$, and given $y = \vecx$, let $\text{mat}(y) = X$. Denote by $c_{kj} \in \mathbb{R}^{nd}$ the vector that has all zeros except in positions $(k-1)n + 1, \ldots kn$ it contains $\frac{1}{n}y_{1j}, \ldots \frac{1}{n}y_{nj}$. Then it is clear that $q_{(k, y_j)}(X) = c_{kj}^{T}\vecx$, so if we let $C \in (\frac{1}{n}\mathcal{Y})^{dl \times dn}$ be the matrix with row $kj \in [dl]$ equal to $c_{kj}$, then $\frac{1}{n}X^{T}Y = C\cdot \vecx$.

\textbf{Proof of Theorem~\ref{thm:ssp-query-acc-label}:}
\begin{proof}
We start with our usual expansion of $f(\hatw)-f(\wstar)$, up until Equation~\ref{eq:penult}, we have with probability $1-\rho$ for every $i \in [l]$:

\begin{equation}
n \cdot f(\hatw)-f(\wstar) =  \tilde{O}\left( \frac{d}{\lambda_{\min} + \lambda} \norm{\wstar}^2 (\normX^4/\epsilon^2)\log(2d^2/\rho) + 
\lambda \norm{\wstar}^2 + \frac{1}{\lambda_{\min} + \lambda}\normreg{E_{2i}}^2_2\right)
\end{equation}
Aggregating over $i$ and rearranging gives: 

\begin{equation}
  \frac{n}{l}\sum_{i = 1}^{l}f(\hatw)-f(\wstar) = \tilde{O}\left( \frac{d}{\lambda_{\min} + \lambda} \norm{\tilde{w}}^2 (\normX^4/\epsilon^2)\log(2d^2/\rho) + 
\lambda \norm{\tilde{w}}^2 + \frac{1}{\lambda_{\min} + \lambda}\frac{1}{l}\sum_{i = 1}^{l}\normreg{E_{2i}}^2_2\right)  , 
\end{equation}
where $\norm{\tilde{w}}^2 = \frac{1}{l}\sum_{i = 1}^{l}\norm{\wstar}^2 = \frac{1}{l}\normreg{W^{*}}_F^2$. Then by Theorem~\ref{thm:proj-acc}, we have with $\frac{1}{l}\sum_{i = 1}^{l}\normreg{E_{2i}}^2_2 = O\left(c(\epsilon, \delta)\sqrt{\log(2/\rho)}n\sqrt{d}\normY^2 \normX^2 \right)$ with probability $1-\rho$. So with probability $1-2\rho$, we have $\frac{n}{l}\sum_{i = 1}^{l}f(\hatw)-f(\wstar) = $
\begin{equation}
  \tilde{O}\left( \frac{d}{\lambda_{\min} + \lambda} \norm{\tilde{w}}^2 (\normX^4/\epsilon^2)\log(2d^2/\rho) + 
\lambda \norm{\tilde{w}}^2 + \frac{1}{\lambda_{\min} + \lambda}c(\epsilon, \delta)\sqrt{\log(2/\rho)}n\sqrt{d}\normY^2 \normX^2\right) 
\end{equation}

Finally optimizing over $\lambda$ gives the desired result: 
$$
\alpha = \tilde{O}\left(\norm{\tilde{w}} \sqrt{\frac{d \norm{\tilde{w}}^2 (\normX^4/\epsilon^2)\log(2d^2/\rho)}{n^2} + \frac{c(\epsilon, \delta)\sqrt{\log(2/\rho)}\sqrt{d}\normY^2 \normX^2}{n}}\right)
$$ 
\end{proof}

\subsection{Proofs from Subsection~\ref{sec:sub_ssp}}
\textbf{Proof of Theorem~\ref{thm:sub_ssp-acc}}:
\begin{proof}
Our analysis will hinge on the case where $l = 1$ e.g. that of standard private linear regression, which we will extend to the PRIMO case by our choice of $\epsilon$ as in the proof of Theorem~\ref{thm:ssp-acc-full}. The fact that the Algorithm is $(O(\epsilon), \delta)$ private follows immediately from the Gaussian mechanism, and the secrecy of the sub-sample lemma (Lemma~\ref{lem:secret}), which is why we can set $\epsilon_1 = \frac{n}{s}\epsilon/2$ in Line $2$. We proceed with the accuracy analysis. 

Define:
\begin{itemize}
    \item $\wstar = (X^{T}X)^{-1}X^{T}Y$: the least squares estimator as before
    \item $\wstar^\lambda = (\frac{1}{n}X^{T}X + \lambda I)^{-1}\frac{1}{n}X^{T}Y$: the ridge regression estimator 
    \item $w_s = (\frac{1}{s}X_S^{T}X_S + \lambda I)^{-1}(\frac{1}{n}X^{T}Y)$: the sub-sampled least squares estimator
    \item $\tilde{w}_s = (\frac{1}{s}X_S^{T}X_S + E_1 + \lambda I)^{-1}(\frac{1}{n}X^{T}Y + E_2)$:  differentially private estimate of $w_s$
\end{itemize}

We note that $0 \leq f(\wstar^\lambda)-f(\wstar) \leq \lambda \left(  \norm{\wstar}^2-\norm{\wstar^\lambda}^2  \right ) \leq \lambda \norm{\mathcal{W}}^2$. Then by the Lemma ~\ref{lem:prederror} and Cauchy-Schwartz with respect to the norm $\normreg{\cdot}_{\frac{X^{T}X}{n}}$:

\begin{multline}\label{eq:error_decomp_ss}
|f(\tilde{w}_s) - f(\wstar)| = \normreg{\tilde{w}_s - \wstar}_{\frac{X^{T}X}{n}}^2 \leq \\
3 \normreg{\wstar - \wstar^{\lambda}}_{\frac{X^{T}X}{n}}^2 + 3 \normreg{\wstar^{\lambda}-w_s}_{\frac{X^{T}X}{n}}^2 + 3 \normreg{w_s - \tilde{w}_s}_{\frac{X^{T}X}{n}}^2 \leq  \\
3 \lambda \norm{\mathcal{W}}^2 +  3||w_s - \wstar^\lambda||_{\frac{X^{T}X}{n}}^2 + 3||\tilde{w}_s - w_s||_{\frac{X^{T}X}{n}}^2  
\end{multline}
Lemmas~\ref{lem:bnd_1}, \ref{lem:bnd_2} bound these terms with high probability. 

\begin{lemma}\label{lem:use_tropp}\label{lem:bnd_1}
Under the assumption $\normX = O(n\lambda)$, then w.p. $1-\rho/2$: 
$$ 
||\wstar^{\lambda} - w_s||_{\frac{X^{T}X}{n}} = O\left(\frac{\normX \normY \log(2d/\rho)}{\lambda s}\right)  
$$
\end{lemma}

\begin{lemma}\label{lem:bnd_2}
Under the assumption $\normX = O(n\lambda)$, then w.p. $1-\rho/2$: 
\begin{equation}
    \normreg{w_s - \tilde{w}_s}_{\frac{X^{T}X}{n}}  = O( \frac{\norm{X}^4}{n^2\epsilon^2}\norm{\mathcal{W}}^2 \cdot \frac{d}{\lambda}\frac{\log(2d/\rho)}{s} +  \frac{l\norm{\mathcal{X}}^2\norm{\mathcal{Y}}^2}{n^2\epsilon^2} \cdot \frac{d}{\lambda}\frac{\log(2d/\rho)}{s})
\end{equation}
\end{lemma} 

Then by Equation~\ref{eq:error_decomp_ss} and Lemmas~\ref{lem:bnd_1}, \ref{lem:bnd_2}, 
we get that w.p. $1-\rho$: 

\begin{multline}
f(\tilde{w}_s) - f(\wstar) = \\
O\left(\lambda \norm{\wstar}^2 + \frac{\normX \normY }{\lambda}(\frac{\log(2d/\rho)}{ s})+  \frac{\norm{X}^4n^2}{s^4\epsilon^2}\norm{\mathcal{W}}^2 \cdot \frac{d}{\lambda}(1+\frac{\log(2d/\rho)}{s}) \; + \; \frac{l\norm{\mathcal{X}}^2\norm{\mathcal{Y}}^2}{n^2\epsilon^2} \cdot \frac{d}{\lambda}(1+\frac{\log(2d/\rho)}{s}) \right), 
\end{multline}
Summing over $i$ and minimizing over $\lambda$ we set $$\lambda = \frac{\sqrt{\normX \normY ( \frac{\log(2d/\rho)}{ s}) + \frac{\norm{X}^4n^2}{s^4\epsilon^2}\norm{\mathcal{W}}^2 \cdot \frac{d}{\lambda}(1+\frac{\log(2d/\rho)}{s}) \; + \; \frac{l\norm{\mathcal{X}}^2\norm{\mathcal{Y}}^2}{n^2\epsilon^2} \cdot \frac{d}{\lambda}(1+\frac{\log(2d/\rho)}{s})}}{\sqrt{1}{l}\normreg{W}_F},$$ which completes the result.
\end{proof}

\textbf{Proof of Lemma~\ref{lem:use_tropp}}: 
\begin{proof}
Now: $$ ||\wstar^{\lambda} - w_s||_{\frac{1}{n}X^{T}X}^2 = n||\wstar^{\lambda} - w_s||_{\frac{1}{n}X^{T}X}^2  \leq n||\wstar^{\lambda} - w_s||_{\frac{1}{n}X^{T}X + \lambda I}^2$$
We will focus on $||\wstar^{\lambda} - w_s||_{\frac{1}{n}X^{T}X + \lambda I}^2$. 
Let $\Sigma = \frac{1}{n}X^{T}X + \lambda I$, $\Sigma_s = \frac{1}{s}\sum_{j \in S}x_s x_s^{T} + \lambda I$, and $v = \frac{1}{n}X^{T}Y$.
Expanding $\normreg{\wstar^{\lambda} - w_s}_{\Sigma} =$
\begin{equation}\label{}
(\Sigma_s^{-1}v - \Sigma^{-1}v)^{T}\Sigma (\Sigma_s^{-1}v - \Sigma^{-1}v) =
v^{T}\overbrace{(\Sigma_s^{-1}-\Sigma^{-1})\Sigma(\Sigma_s^{-1}-\Sigma^{-1})}^{A}v = v^{T}Av
\end{equation}
Now since $A$ is Hermitian, we know $||v||_{A} \leq ||v|| ||A||_{2}$. Since $||v|| \leq \normX \normY$, it suffices to bound $\normreg{A}_{2}$ with high probability. Noting that $A = (\Sigma_s^{-1}-\Sigma^{-1})\Sigma(\Sigma_s^{-1}-\Sigma^{-1}) = 
\Sigma_s^{-1}\Sigma^{1/2}(I - \Sigma^{-1/2}\Sigma_s\Sigma^{-1/2})(I - \Sigma^{-1/2}\Sigma_s\Sigma^{-1/2})\Sigma^{1/2}\Sigma_s^{-1}$, we have by the sub-multiplicativity of the operator norm:

\begin{multline}\label{eq:a_decomp}
||A||_{2} \leq \left( \normreg{\Sigma_S^{-1}\Sigma^{1/2}}_{2}^2 \right) \cdot \left(\normreg{I - \Sigma^{-1/2}\Sigma_s\Sigma^{-1/2}}_{2}^2 \right) \leq \\ 
\frac{\lambda_{\max}(\Sigma)}{\lambda^2} \cdot \left(\normreg{I - \Sigma^{-1/2}\Sigma_s\Sigma^{-1/2}}_{2}^2 \right)
\end{multline}

Now consider $\Sigma^{-1/2}\Sigma_s\Sigma^{-1/2} = \Sigma^{-1/2}\frac{1}{s}\sum_{j \in S}(x_s x_s^{T} + \lambda I)\Sigma^{-1/2} = \frac{1}{s}\sum_{j \in s}Z_i$. Then note that $\mathbb{E}[Z_i] = \Sigma^{-1/2}(\frac{1}{n}\mathbb{E}[x_ix_i^{T}] + \lambda I)\Sigma^{-1/2} = \Sigma^{-1/2}\Sigma\Sigma^{-1/2} = I$, and that $\lambda_{max}(Z_i) \leq \normreg{\Sigma^{-1}}_2 \normreg{\frac{1}{n}x_ix_i^{T} + \lambda I}_2 \leq 1 + \frac{ \normX}{n\lambda}$.

Now we can bound $\normreg{I-\frac{1}{s}\sum_{j \in s}Z_i}_2$ by Theorem~2.2 in \cite{tropp}:

\begin{lemma}{\cite{tropp}}
\label{lem:tropp1}
Let $Z_1, \ldots Z_s$ sampled without replacement from $\{Z_1, \ldots Z_n\}$. Then if $Z_i \in \mathbb{S}_{d}^{+}, \mathbb{E}[Z_i] = I_{d}$, and $\max_{i \in [n]}\lambda_{\max}(Z_i) \leq B$ w.p. $1-\rho$, for $\delta = \sqrt{\frac{2B\log(2d/\rho)}{s}}$:
\begin{align*}
\lambda_{\max}(\frac{1}{s}\sum_{i = 1}^{s}Z_i) < 1 + \delta, \; \; \lambda_{\min}(\frac{1}{s}\sum_{i = 1}^{s}Z_i)   > 1 - \delta 
\end{align*}
\end{lemma}
So by Lemma~\ref{lem:tropp1}, we know that with probability $1-\rho$: $|\lambda_{\min}(I - \Sigma^{-1/2}\Sigma_s\Sigma^{-1/2})|  = 1 - \lambda_{\max}(\Sigma^{-1/2}\Sigma_s\Sigma^{-1/2}) \leq \delta$, and similarly $|\lambda_{\max}(I - \Sigma^{-1/2}\Sigma_s\Sigma^{-1/2})| \leq \delta$, thus $||I -\Sigma^{-1/2}\Sigma_s\Sigma^{-1/2})||_2 \leq \delta$. Substituting this all into Equation~\ref{eq:a_decomp} and noting $\lambda_{\max}(\Sigma) \leq \frac{\normX}{n} + \lambda$ we get with probability $1-\rho$: 
\begin{align}
\normreg{\wstar^\lambda - w_s}_{\frac{1}{n}X^{T}X} \leq \normreg{\tilde{w}_s - w_s}_{\Sigma} \leq \normX \normY n \normreg{A}_{2} \leq  \normX \normY \frac{\lambda_{\max}(\Sigma)}{\lambda^2}\delta^2 = \\
2\normX \normY \frac{(\frac{\normX}{n} + \lambda)^2}{\lambda^3} \frac{\log(2d/\rho)}{s},
\end{align}
which under the assumption $\normX = O(n\lambda)$ gives $O(\frac{\normX \normY \log(2d/\rho)}{\lambda s})$.
as desired.
\end{proof}

\textbf{Proof of Lemma~\ref{lem:bnd_2}:}
\begin{proof}
By Lemma~\ref{lem:sherman}, with $A = \frac{1}{s}X_S^{T}X + \lambda I$, $B = E_1, c = E_2, v = \frac{1}{n}^{T}Y$, we get $w_s - \tilde{w}_s = (\frac{1}{s}X_S^{T}X + \lambda I + E_1)E_1w_s - (\frac{1}{s}X_S^{T}X + \lambda I + E_1)^{-1}E_2$, and so 
$$
\normreg{w_s - \tilde{w}_s}_{\frac{X^{T}X}{n}}^2 \leq 2\normreg{(\frac{1}{s}X_S^{T}X + \lambda I + E_1)^{-1}E_1w_s}_{\frac{X^{T}X}{n}}^2 + 2\normreg{(\frac{1}{s}X_S^{T}X + \lambda I + E_1)^{-1}E_2}_{\frac{X^{T}X}{n}}^2
$$
Under the assumption $||E_1||_{2} \leq \lambda/2,$  this becomes 
\begin{multline}
  \normreg{w_s - \tilde{w}_s}_{\frac{X^{T}X}{n}}^2 = O\left(\normreg{E_1 w_s}_{(\covs + \lambda I)^{-1}(X^{T}X/n + \lambda I)(\covs + \lambda I)^{-1}}\right) + \\
  O\left(\normreg{E_2}_{(\covs + \lambda I)^{-1}(X^{T}X/n + \lambda I)(\covs + \lambda I)^{-1}}\right)
\end{multline}
Now to apply Lemma~\ref{lem:jl}, we need to bound
\begin{multline}
    \text{Tr}((\covs + \lambda I)^{-1}(X^{T}X/n + \lambda I)(\covs + \lambda I)^{-1}) \leq d \lambda_{\max}(\Sigma_S)^{-1}(\Sigma)(\Sigma_S)^{-1}) = \\
    d\lambda_{\max}(\Sigma^{-1/2}(\Sigma^{1/2}\Sigma_S^{-1}\Sigma^{1/2})^2\Sigma^{-1/2})\leq \\
    d\lambda_{\max}(\Sigma^{-1})\frac{1}{\lambda_{\min}(\Sigma^{-1/2}\Sigma^{1/2}\Sigma_S^{-1}\Sigma^{1/2})}^2 \leq \frac{d}{\lambda}(\frac{1}{1-\delta})^2,
\end{multline}
where the last inequality follows from Lemma~\ref{lem:tropp1}. Applying Lemma~\ref{lem:jl} we get that with probability $1-2\rho$:
\begin{multline}
    ||\wstar^{\lambda} - w_s||_{X^{T}X/n} = O\left( \sigma_1^2 \cdot \frac{d}{\lambda}(\frac{1}{1-\delta})^2 \cdot \norm{w_s}^2\log(2d^2/\rho) + \sigma_2^2 \cdot \frac{d}{\lambda}(\frac{1}{1-\delta})^2\log(2d^2/\rho)\right),
\end{multline}

From Lemma~\ref{lem:tropp1}, $\delta = \sqrt{\frac{2(1+ \frac{\normX}{n\lambda})\log(2d/\rho)}{s}}$, which under the assumption $\normX = O\left(n\lambda\right)$ gives $\frac{1}{(1-\delta)}^2 = O(1 + \frac{\log(2d/\rho)}{s})^2$. Substituting in the value of $\sigma_1, \sigma_2$ gives: 

$$
 \frac{\norm{X}^4n^2}{s^4\epsilon^2}\norm{\mathcal{W}}^2 \cdot \frac{d}{\lambda}(1+\frac{\log(2d/\rho)}{s}) \; + \; \frac{l\norm{\mathcal{X}}^2\norm{\mathcal{Y}}^2}{n^2\epsilon^2} \cdot \frac{d}{\lambda}(1+\frac{\log(2d/\rho)}{s}),
$$ 
as desired.
\end{proof}

\subsection{Computational Complexity}
\label{sec:complex}
We continue the discussion of how to efficiently compute the projection step in Algorithm~\ref{alg:proj} by exploiting properties of the Kronecker product. Recall that $C = I_d \otimes \frac{1}{n}Y_T$ where $\otimes$ denotes the Kronecker product. Then if $L \Lambda V^T$ denotes the SVD of $Y^T$, standard properties of the Kronecker product imply that the spectral decomposition of $C^TC$ is: 

\begin{align}\label{eq:kron}
C &= \frac{1}{n} \cdot (I_d \otimes L)(I_d \otimes \Lambda)(I_d \otimes V^{T}) \implies \\
C^{T}C &= (I_d \otimes V)(I_d \otimes \Lambda^2)(I_d \otimes V^{T})
\end{align}
Hence we can compute $\text{SPEC}(C)$ in the time it takes to compute $\text{SVD}(Y)$, or $O(nl \min(n,l))$. Similarly, to efficiently compute the $U^{T}b$ term required for Lemma~2.2 \cite{hager} we can again take advantage of properties of the Kronecker product, $U^{T}b = $
\begin{align*}
  (I_d \otimes V)^{T}2C^{T}\tilde{g} &= 2(I_d \otimes V^{T})(I_d \otimes \frac{1}{n}Y)\tilde{g}  = \\
  2(I_d \otimes \frac{1}{n}V^{T}y)\tilde{g} &= \text{vec}(\frac{2}{n}V^{T}Y\text{mat}(\tilde{g})),
\end{align*}
where $\text{mat}(\tilde{g})$ is the $l \times d$ matrix with row $i$ given by elements $(l(i-1)+1, \ldots, l(i-1) + l$ of $\tilde{g}$, and the last equality follows properties of the Kronecker product. Now $V^{T}Y = \Lambda U^{T}$ which can be computed in $O(ln)$ since $\Lambda$ is diagonal. Multiplying by $\text{mat}(\tilde{g})$ can be done in another $O(nld)$, for total complexity of $O(nl\max(\min(n,l),d)).$

\begin{algorithm}[H]{$\lambda-\algotext{ReuseCov}$}
\caption{Input:  $\lambda, \mathcal{X} \in \mathcal{X}^n \subset \mathbb{R}^{d \times n}$, $Y = [y_1, \ldots y_{l}] \in \mathcal{Y}^{l \times n},$ privacy params: $\epsilon, \delta$. We denote by $\mathcal{B}$ the Algorithm in Lemma 2.2 \cite{hager}}
\label{alg:reusecov_emp}
\begin{algorithmic}[1]
\State Draw $E_1 \sim N_{d(d+1)/2}(0, \sigma_1^2),$ where $\sigma_1 =  \frac{1}{n} 2\sqrt{2 \log(2.5/\delta)}\normX^2/\epsilon$
\State Compute $\hat{I} = (\frac{1}{n}X^{T}X + E_1 + \lambda I)$
\State Compute the QR decomposition $\hat{I} = QR$
\State Draw $\hat{v} = [\hat{v}_1, \ldots \hat{v}_l] \sim \algotext{GaussMech}(\epsilon/2, \delta/2, \Delta =  \frac{1}{n}\sqrt{l}\normX \normY)$
\State Compute $\text{SVD}(Y^{T}) = U\Lambda V^{T}$
\State Compute $\hat{g} = \Pi_{C(Y)}\hat{v} = \mathcal{B}(V, \Lambda, \hat{v})$
\For {$i = 1 \ldots l$}
    \State Solve $R\hat{w_i} = Q^{T}\hat{g}_i$ by back substitution.
\EndFor   
\State Return $\hat{W} = [\hat{w_1}, \ldots \hat{w_{l}}]$
\end{algorithmic}
\end{algorithm}

\subsection{Sketching \algotext{ReuseCov}}
\label{sec:sub_ssp}
Theorem~\ref{thm:complex} shows that when $n > d > l$, the complexity of Algorithm~\ref{alg:reusecov_emp} is $O(nd^2)$ or the cost of forming the covariance matrix. In this section we show how sub-sampling $s < n$ points can improve this to $O(sd^2)$ by giving an analysis of sub-sampled SSP. The key ingredient is marrying the convergence of the sub-sampled covariance matrix to $\covx$ with the accuracy analysis of SSP we saw in Section~\ref{sec:ssp}.

Our algorithm is based on the observation that if we sub-sample $S \subset [n], |S| = s$ points without replacement then: 
\begin{itemize}
    \item The cost of computing the covariance matrix $\Sigma_S = \sum_{k \in S}x_k x_k^{T}$ is $O(sd^2)$
    \item By the ``secrecy of the sub-sample" principle \cite{privacybook}, our privacy cost for estimating $\Sigma_S$ is scaled down by a factor of $s/n$
    \item With high probability for sufficiently large $s$, $\Sigma_S \to \Sigma$ by a matrix-Chernoff bound for sampling without replacement \cite{tropp}
\end{itemize}

 \begin{algorithm}[H]{$\lambda-\algotext{SubSampReuseCov}$}
\caption{Input:  $s, \lambda, \mathcal{X} \in \mathcal{X}^n \subset \mathbb{R}^{d \times n}$, $Y = [y_1, \ldots y_{l}] \in \mathcal{Y}^{l \times n},$ privacy params: $\epsilon, \delta$}
\label{alg:sub_ssp}
\begin{algorithmic}[1]
\State Sub-sample $s$ points without replacement from $\mathcal{X}$, we denote the sub-sampled design matrix by $X_S$. 
\State Let $(\epsilon_1, \delta_1) = (\frac{n}{s}\epsilon/2, \delta/2)$
\State Draw $E_1 \sim N_{d(d+1)/2}(0, \sigma_1^2),$ where $\sigma_1 =  \frac{1}{n} 2\sqrt{2 \log(2.5/\delta)}\normX^2/\epsilon_1$
\State Compute $\hat{I_s} = (\frac{1}{s}X_S^{T}X_S + E_1 + \lambda I)$
\State Draw $\hat{v} = [\hat{v}_1, \ldots \hat{v}_l] \sim \algotext{GaussMech}(\epsilon/2, \delta/2, \Delta =  \frac{1}{n}\sqrt{l}\normX \normY)$
\For {$i = 1 \ldots l$}
    \State Set $\hat{w_i} = \hat{I_s}^{-1}\hat{v}_i$
\EndFor   
\State Return $\hat{W} = [\hat{w_1}, \ldots \hat{w_{l}}]$
\end{algorithmic}
\end{algorithm}

\begin{theorem}
\label{thm:sub_ssp-acc}
With $\M = \algotext{GaussMech}(\epsilon/2, \delta/2, \Delta =  \frac{1}{n}\sqrt{l}\normX \normY)$, Algorithm~\ref{alg:sub_ssp} is an $(\alpha, \rho, O(\epsilon), \delta)$ solution to the PRIMO problem with $\alpha^2 = $
\begin{equation} 
\begin{split}
 O(\norm{\hat{w}}^2\normX^2 \normY^2 ( \frac{\log(2d/\rho)}{ s}) + \frac{d}{\lambda}(1+\frac{\log(2d/\rho)}{s}) \cdot  \\
 (\frac{\norm{X}^4n^2}{s^4\epsilon^2}\norm{\mathcal{W}}^2 + \frac{l\norm{\mathcal{X}}^2\norm{\mathcal{Y}}^2}{n^2\epsilon^2}))
\end{split}
\end{equation}
\end{theorem}

\paragraph{Proof Sketch.}
Our analysis will hinge on the case where $l = 1$ e.g. that of standard private linear regression, which we will extend to the PRIMO case by our choice of $\epsilon$ as in the proof of Theorem~\ref{thm:ssp-acc-full}. Now let:
\begin{itemize}
    \item $\wstar = (X^{T}X)^{-1}X^{T}Y$ the least squares estimator as before
    \item $\wstar^\lambda = (\frac{1}{n}X^{T}X + \lambda I)^{-1}\frac{1}{n}X^{T}Y$ the ridge regression estimator 
    \item $w_s = (\frac{1}{s}X_S^{T}X_S + \lambda I)^{-1}(\frac{1}{n}X^{T}Y)$ the sub-sampled least squares estimator
    \item $\tilde{w}_s = (\frac{1}{s}X_S^{T}X_S + E_1 + \lambda I)^{-1}(\frac{1}{n}X^{T}Y + E_2)$ our differentially private estimate of $w_s$
\end{itemize}

We note that $0 \leq f(\wstar^\lambda)-f(\wstar) \leq \lambda \left(  \norm{\wstar}^2-\norm{\wstar^\lambda}^2  \right ) \leq \lambda \norm{\mathcal{W}}^2$. Then by the Lemma ~\ref{lem:prederror} and Cauchy-Schwartz with respect to the norm $\normreg{\cdot}_{\frac{X^{T}X}{n}}$:

\begin{multline}\label{eq:error_decomp_ss}
|f(\tilde{w}_s) - f(\wstar)| = \normreg{\tilde{w}_s - \wstar}_{\frac{X^{T}X}{n}}^2 \leq \\
3 \normreg{\wstar - \wstar^{\lambda}}_{\frac{X^{T}X}{n}}^2 + 3 \normreg{\wstar^{\lambda}-w_s}_{\frac{X^{T}X}{n}}^2 + 3 \normreg{w_s - \tilde{w}_s}_{\frac{X^{T}X}{n}}^2 \leq  \\
3 \lambda \normreg{\mathcal{W}}^2 +  \underbrace{3||w_s - \wstar^\lambda||_{\frac{X^{T}X}{n}}^2}_{\text{Matrix Chernoff}} + \underbrace{3||\tilde{w}_s - w_s||_{\frac{X^{T}X}{n}}^2 }_{\text{SSP analysis} + \text{Matrix Chernoff}}
\end{multline}

So it suffices to bound each term with high probability. The second term, $||\tilde{w}_s - w_s||_{X^{T}X + \lambda I}$ can be bounded using the same arguments as in Theorem~\ref{thm:ssp-acc-full}, with small differences due to scaling. Crucially though, as we need to bound this in the norm induced by $X^{T}X$ rather than $X_S^{T}X_S$, we will need to utilize the convergence of $X_S^{T}X_S \to \covx$ via Matrix-Chernoff bounds. 

\begin{lemma}\label{lem:bnd_2}
Under the assumption $\normX = O(n\lambda)$, then w.p. $1-\rho/2$: 
\begin{equation}
\label{eq:sub}
    \normreg{w_s - \tilde{w}_s}_{\frac{X^{T}X}{n}}  = O( \frac{\norm{X}^4}{n^2\epsilon^2}\norm{\mathcal{W}}^2 \cdot \frac{d}{\lambda}\frac{\log(2d/\rho)}{s} +  \frac{l\norm{\mathcal{X}}^2\norm{\mathcal{Y}}^2}{n^2\epsilon^2} \cdot \frac{d}{\lambda}\frac{\log(2d/\rho)}{s})
\end{equation}
\end{lemma} 

Bounding the first term can be reduced to bounding $\normreg{I - (\frac{1}{n}\covx)^{-1/2}(\frac{1}{s}X_S^{T}X_S)(\frac{1}{n}\covx)^{-1/2}}_2$ which follows  more directly via  the Matrix-Chernoff bound for sub-sampling without replacement:
\begin{lemma}\label{lem:use_tropp}\label{lem:bnd_1}
Under the assumption $\normX = O(n\lambda)$, then w.p. $1-\rho/2$: 
$$ 
||\wstar^{\lambda} - w_s||_{\frac{X^{T}X}{n}} = O\left(\frac{\normX \normY \log(2d/\rho}{\lambda s}\right)  
$$
\end{lemma}

Substituting into Equation~\ref{eq:sub} and minimizing over $\lambda$ gives the desired result.

\subsection{Additional Experimental Results}
\label{sec:app_exp}

\textbf{Realizability of outcomes has a minimal effect.}

\begin{figure}[!h]
\centering
\begin{subfigure}[c]{.45\textwidth}
\includegraphics[width=\textwidth]{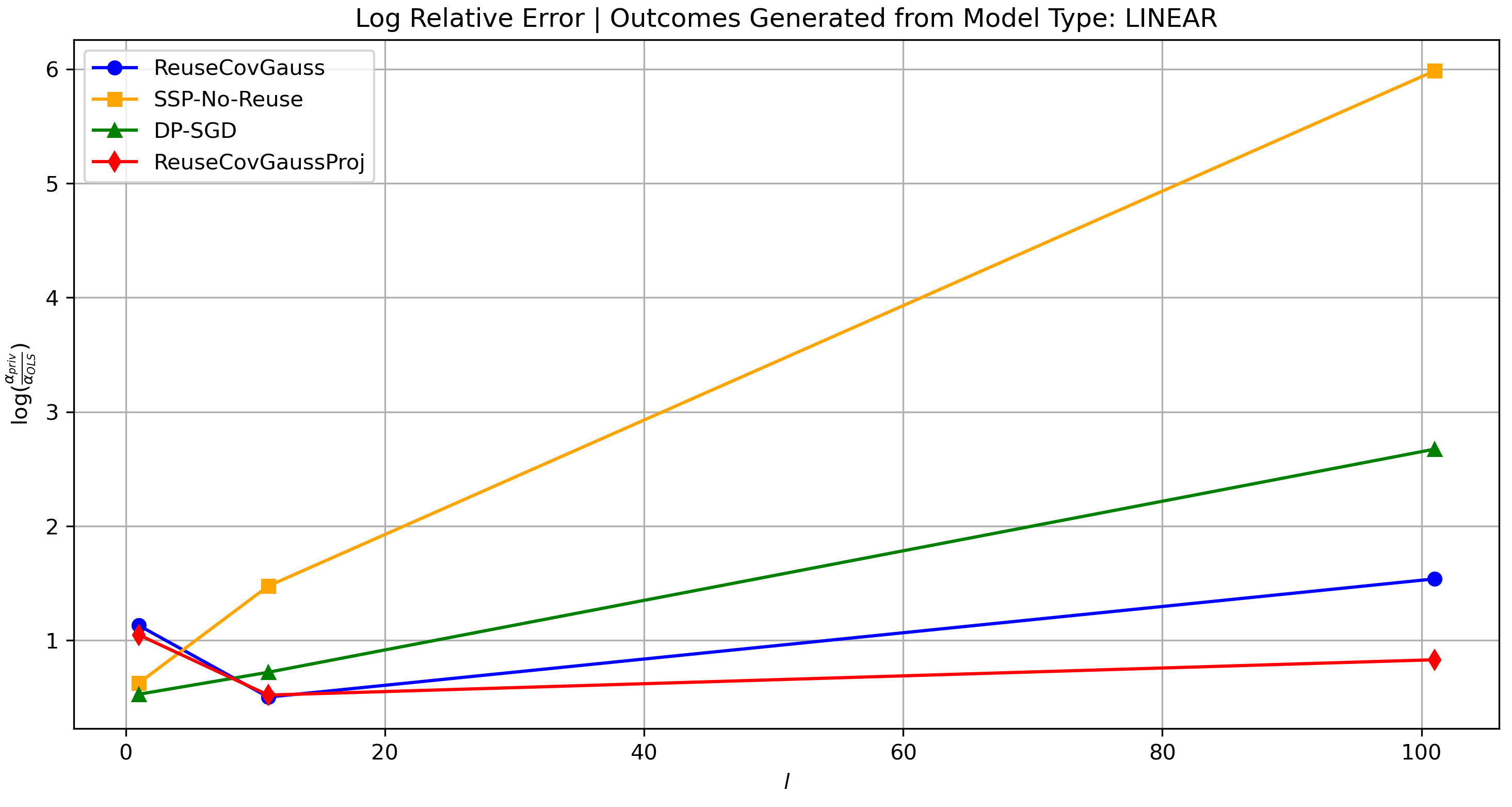}
\caption{}
\end{subfigure}
\begin{subfigure}[c]{.45\textwidth}
\includegraphics[width=\textwidth]{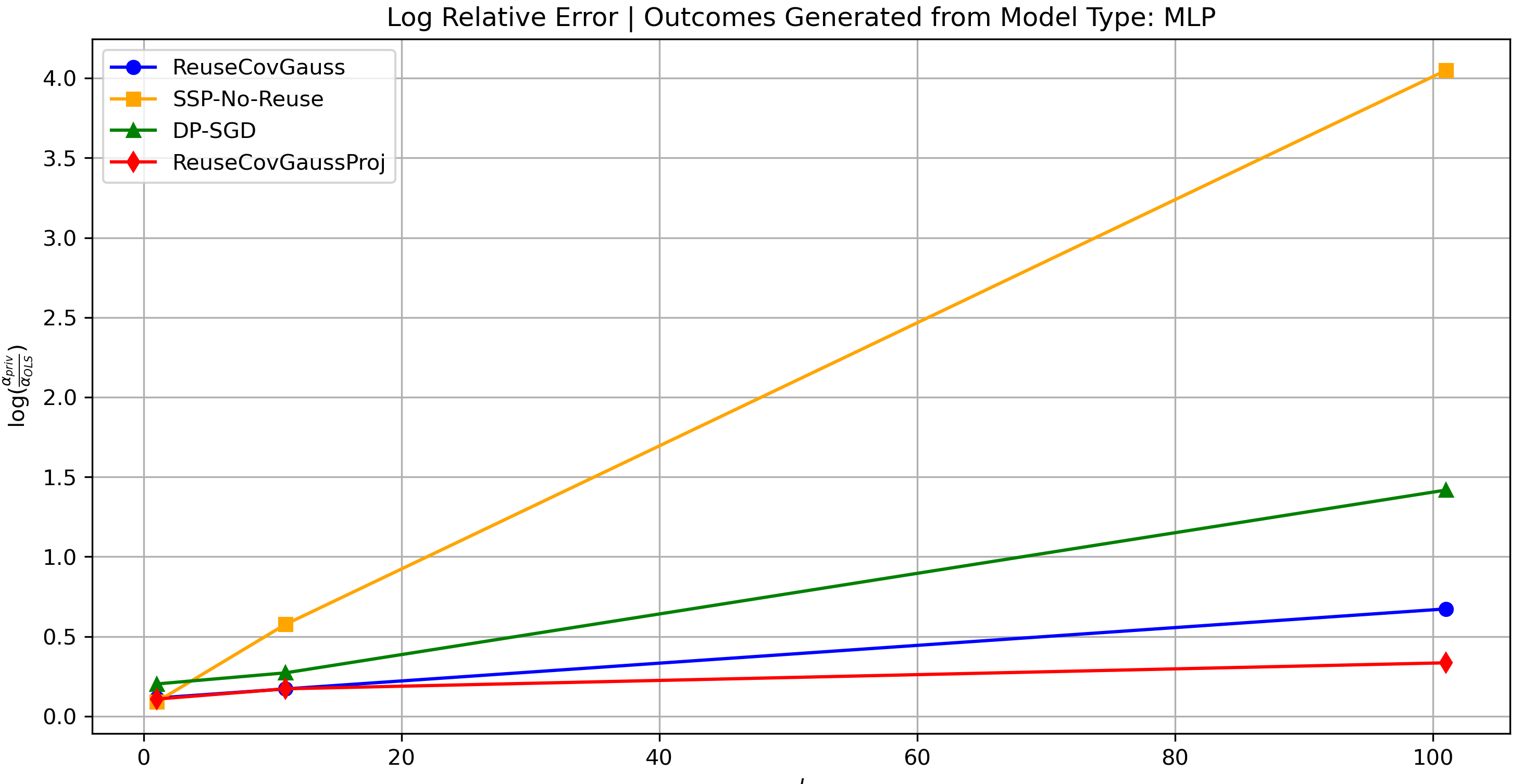}
\caption{}
\end{subfigure}
\caption{Both (a) and (b) show the log of the ratio of the squared loss of the private estimator to the square loss of the OLS estimator, but in (a) the outcomes are synthetically generated from a linear model, whereas in (b) the outcomes are generated from a 2-layer neural network, to test how the algorithms perform when the outcomes are not generated by a linear model. All values are averaged over $5$ iterations, over $l =  (1, 11, 101)$ with $(\epsilon, \delta) = (5,\frac{1}{n^2})$ and $n = 5000, d  = 25$ on a synthetic dataset with Gaussian features.}
\label{fig:mlp_vs_lin}
\end{figure}

\textbf{DP-SGD outperforms \algotext{ReuseCov} for $l < \frac{n}{\sqrt{d}}$}. 
We construct our dataset by sub-sampling $n = 5000$ images of dimension $d = 784$ from MNIST \cite{mnist}, and generating $l = 1, 11, 101$ outcomes from a noisy linear model with unit norm. We implement DP-SGD using the Opacus \cite{opacus} library from Meta, and do privacy accounting across regressions by composing in zCDP \cite{} and then converting back to $(\epsilon, \delta)-\text{DP}$. Figure~\ref{fig:mnist} shows that \algotext{DP-SGD} outperforms our \algotext{ReuseCov} algorithms when $l \in (1, 11, 101)$, although, in agreement with the theory, this gap shrinks as $l$ increases. 

\begin{figure}
    \centering
    \includegraphics[width=.55\textwidth]{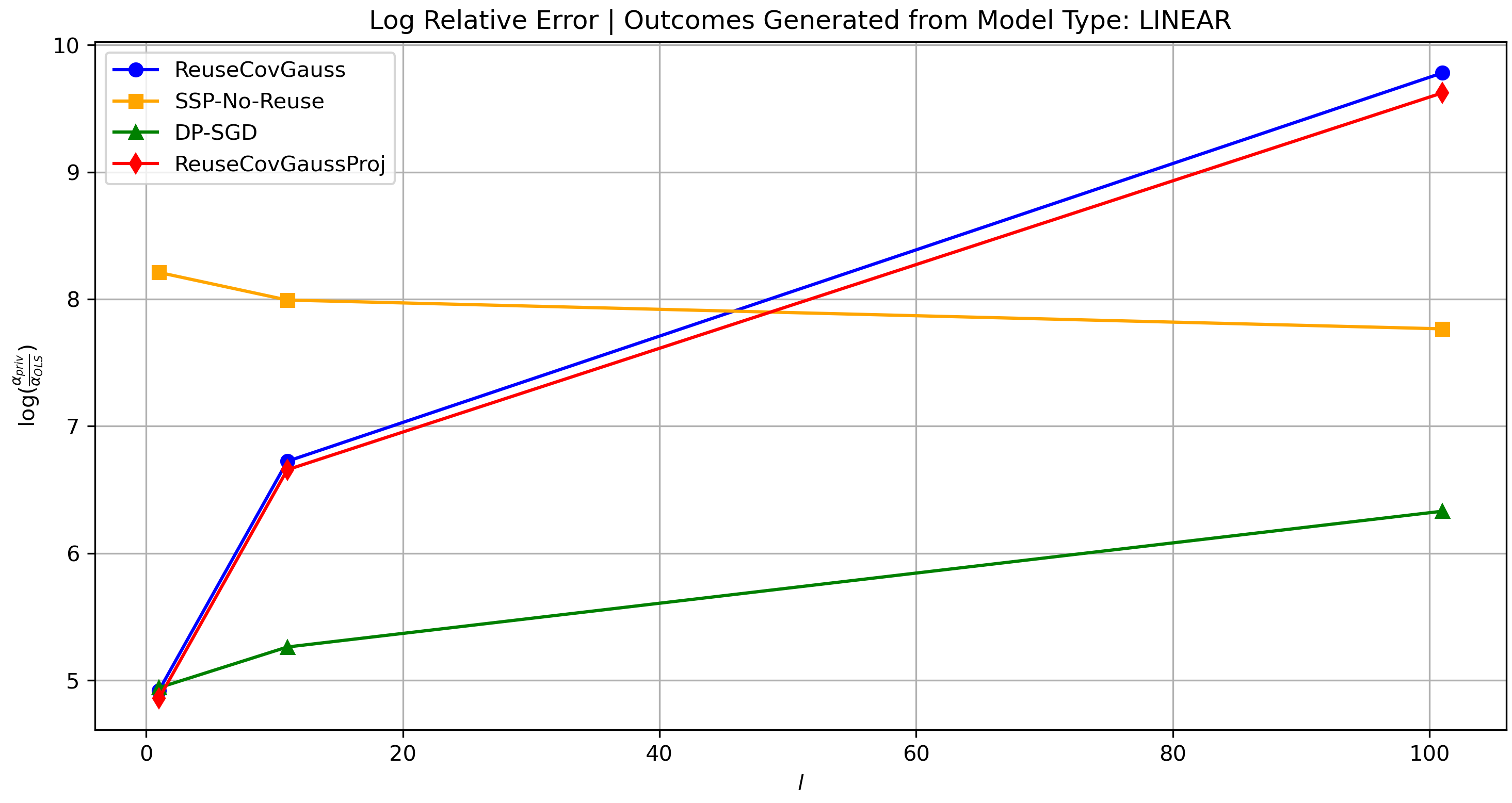}
    \caption{We construct our MNIST dataset by sampling $n = 5000$ random training images $(d = 784)$, and then generating synthetic labels from a random linear model. All values are averaged over $5$ iterations, over $l =  (1, 11, 101)$ with $(\epsilon, \delta) = (5,\frac{1}{n^2})$.}
    \label{fig:mnist}
\end{figure}

\end{document}